\documentclass{llncs}
\usepackage{makeidx}  

\usepackage{amssymb,amsmath, mathrsfs}
\usepackage{graphicx}
\usepackage{textgreek}
\usepackage{wrapfig}

\usepackage{algorithm}
\usepackage[noend]{algpseudocode}

\usepackage{url}

\DeclareMathOperator{\Ima}{Im}

\newcommand{\q}{\ensuremath{\mathbb{Q}}}

\long\def\ignore#1{}
\def\myps[#1]#2{\includegraphics[#1]{#2}}

\def\br(#1,#2){{\langle #1,#2 \rangle}}
\def\setZ[#1,#2]{{[ #1 .. #2 ]}}

\def\q={\quad=\quad}
\def\qq={\qquad=\qquad}

\def\psfile[#1]#2{}
\def\psfilehere[#1]#2{}

\def\assign(#1,#2){\langle#1,#2\rangle}
\def\edge(#1,#2){(#1,#2)}

\def\slack(#1){\texttt{slack}({#1})}
\def\barslack(#1){\overline{\texttt{slack}}({#1})}

\def\unitvec(#1){{{\bf u}_{#1}}}

\makeatletter
\def\thm@space@setup{%
  \thm@preskip=2pt \thm@postskip=2pt
}
\makeatother

\newif\ifTR
\TRfalse

\newif\ifCOM
\COMfalse

\begin{document}
\title{Fourier analysis perspective for sufficient dimension reduction problem}
\titlerunning{PCA}  
%
\author{Rustem Takhanov}
\authorrunning{Rustem Takhanov} 
%
%
\institute{Nazarbayev University\\
\email{rustem.takhanov@nu.edu.kz},\\
}

\maketitle              

\begin{abstract}
A theory of sufficient dimension reduction (SDR) is developed from an optimizational perspective.
In our formulation of the problem, instead of dealing with raw data, we assume that our ground truth includes a mapping ${\mathbf f}: {\mathbb R}^n\rightarrow {\mathbb R}^m$ and a probability distribution function $p$ over ${\mathbb R}^n$, both given analytically. We formulate SDR as a problem of finding a function ${\mathbf g}: {\mathbb R}^k\rightarrow {\mathbb R}^m$ and a matrix $P\in {\mathbb R}^{k\times n}$ such that ${\mathbb E}_{{\mathbf x}\sim p({\mathbf x})} \left|{\mathbf f}({\mathbf x}) - {\mathbf g}(P{\mathbf x})\right|^2$ is minimal. It turns out that the latter problem allows a reformulation in the dual space, i.e. instead of searching for ${\mathbf g}(P{\mathbf x})$ we suggest searching for its Fourier transform. First, we characterize all tempered distributions that can serve as the Fourier transform of such functions.
The reformulation in the dual space can be interpreted as a problem of finding a $k$-dimensional linear subspace $S$ and a tempered distribution ${\mathbf t}$ supported in $S$ such that ${\mathbf t}$ is ``close'' in a certain sense to the Fourier transform of ${\mathbf f}$. 

Instead of optimizing over generalized functions with a $k$-dimensional support, we suggest minimizing over ordinary functions but with an additional term $R$ that penalizes a strong distortion of the support from any $k$-dimensional linear subspace. For a specific case of $R$, we develop an algorithm that can be formulated for functions given in the initial form as well as for their Fourier transforms. Eventually, we report results of numerical experiments with a discretized version of the latter algorithm.
\end{abstract}


\section{Introduction}
{\em The dimensionality reduction} is an important problem in data science that has many facets and non-equivalent formulations coming from different contexts, either purely mathematical or appearing in applications. The classical one was first formulated in the work of R. Fisher~\cite{Fisher309} and currently known as {\em the principal component analysis}. Subsequently, the idea of principal components was applied to more general frameworks, giving birth to new branches of statistics/machine learning such as the manifold learning (e.g. the nonlinear dimensionality reduction) and the sufficient dimension reduction. In the manifold learning formulation (which is the direct generalization of the classical) we are usually given a finite number of points in ${\mathbb R}^n$ (sampled according to some unknown distribution) and our goal is to find a ``low-dimensional'' geometric structure that approximates ``the support'' of the distribution and satisfies some additional properties such as smoothness, low complexity etc.

Unlike the latter formulations, in the sufficient dimension reduction (sometimes called the supervised dimension reduction), we are given a finite number of pairs $({\mathbf x}_i, {\mathbf y}_i), {\mathbf x}_i\in {\mathbb R}^n, {\mathbf y}_i\in {\mathbb R}^m$, also generated according to some unknown joint distribution $p({\mathbf x}, {\mathbf y})$, and our goal is to find $k$ vectors (where $k<<n$) ${\mathbf w}_1, \cdots, {\mathbf w}_k\in {\mathbb R}^n$ such that symbolically:
$$
{\mathbf y}\perp\!\!\!\perp {\mathbf x} | {\mathbf w}^T_1{\mathbf x}, \cdots, {\mathbf w}^T_k{\mathbf x}
$$
The latter means that an output ${\mathbf y}$ is conditionally independent of ${\mathbf x}$, given ${\mathbf w}^T_1{\mathbf x}, \cdots, {\mathbf w}^T_k{\mathbf x}$. Or, that conditional distribution $p({\mathbf y}| {\mathbf x})$ is the same as $p({\mathbf y}| {\mathbf w}^T_1{\mathbf x}, $ $ \cdots, {\mathbf w}^T_k{\mathbf x})$. 

Of course, the latter formulation can hardly be solved if we do not make any assumptions on the joint distribution, or more specifically on the conditional distribution $p({\mathbf y}| {\mathbf x})$. A standard assumption is the following semi-parametric discriminative model:
\begin{equation}\label{discr}
{\mathbf y} = {\mathbf g}({\mathbf w}^T_1{\mathbf x}, \cdots, {\mathbf w}^T_k{\mathbf x}) + \text{\boldmath$\varepsilon$}
\end{equation}
where $\text{\boldmath$\varepsilon$}$ is a Gaussian noise with ${\mathbb E} \text{\boldmath$\varepsilon$} = {\mathbf 0}$ and ${\mathbb E} \text{\boldmath$\varepsilon$}\text{\boldmath$\varepsilon$}^T = \delta^2 I$. The function ${\mathbf g}$ is an unknown smooth function. Then, the function ${\mathbf f}({\mathbf x}) = {\mathbb E} [{\mathbf y} | {\mathbf x}] = {\mathbf g}({\mathbf w}^T_1{\mathbf x}, \cdots, {\mathbf w}^T_k{\mathbf x})$ is called the regression function. 

There are 3 major methods to estimate parameters of model~\ref{discr}: (1) sliced inverse regression~\cite{Li91},\cite{Cook91}; (2) methods based on an analysis of gradient and Hessian of the regression function~\cite{Li2}, \cite{Xia}, \cite{Mukherjee}; (3) methods based on combining local classifiers~\cite{Hastie}, \cite{Sugiyama}.

Probably, the closest to ours is the second approach. Let us briefly outline its idea for $m=1$. According to that approach we first recover the regression function $f$ and estimate the distribution $p({\mathbf x})$ from our data $\{({\mathbf x}_i, y_i)\}$. The former can be done by solving the supervised learning problem using any suitable model, e.g. by neural networks, and the latter is typically done by assuming that $p({\mathbf x}) = \frac{1}{\sqrt{(2\pi)^n|\Sigma|}}e^{-\frac{1}{2}({\mathbf x}-\text{\boldmath$\mu$})^T\Sigma^{-1}({\mathbf x}-\text{\boldmath$\mu$})}$ and estimating the parameters $\text{\boldmath$\mu$}, \Sigma$ of the multivariate normal distribution.  
At the second stage we no longer need our data and treat $f, p$ as {\em the ground truth}. Since, for recovered $f$ it is natural to expect that $f({\mathbf x}) \approx g({\mathbf w}^T_1{\mathbf x}, \cdots, {\mathbf w}^T_k{\mathbf x})$, then a natural way to reconstruct vectors  ${\mathbf w}_1, \cdots, {\mathbf w}_k$ is to set them equal to first $k$ principal components of the matrix ${\mathbb E}_{{\mathbf x}\sim p({\mathbf x})} H_{{\mathbf f}}({\mathbf x})$, where $H_{{\mathbf f}}({\mathbf x}) = \begin{bmatrix}
\frac{\partial^2 f}{\partial x_i \partial x_j}
\end{bmatrix}$ is a Hessian matrix of $f$ at point ${\mathbf x}$.

In our paper we also assume that ${\mathbf f}, p$ is an already given ground truth, though unlike the previous approach, we formulate the main problem optimizationally, i.e. our goal is to find 
\begin{equation}
{\mathbb E}_{{\mathbf x}\sim p({\mathbf x})} \left|{\mathbf f}({\mathbf x}) - {\mathbf g}({\mathbf w}^T_1{\mathbf x}, \cdots, {\mathbf w}^T_k{\mathbf x})\right|^2 \rightarrow \min_{{\mathbf g}, {\mathbf w}_1, \cdots, {\mathbf w}_k}
\end{equation}
It is easy to see that the latter corresponds to the maximum likelihood approach to estimating of the parameters ${\mathbf w}_1, \cdots, {\mathbf w}_k, {\mathbf g}$. Since ${\mathbf g}$ is an infinite-dimensional object, we analyse it by the tools of functional analysis, specifically using a theory of tempered distributions. The key observation of our analysis, stated in theorem~\ref{basic} of section~\ref{I}, is that a class of functions of the form ${\mathbf g}({\mathbf w}^T_1{\mathbf x}, \cdots, {\mathbf w}^T_k{\mathbf x})$ can be characterized as those functions whose Fourier transform is supported in a $k$-dimensional linear subspace. Instead of optimizing over generalized functions with a $k$-dimensional support, we suggest minimizing over ordinary functions given in a generic form but with an additional constraint. In order to force their support to be $k$-dimensional, in section~\ref{II} we introduce a class of penalty functions $R$ such that large values of $R$ indicate a strong distortion of the support from any $k$-dimensional linear subspace. For a specific case of $R$, in section~\ref{III} we develop an algorithm for our problem that can be formulated for functions given in the frequency coordinate form as well as in the initial coordinate form. The last section is dedicated to experiments on synthetic data.
\section{Preliminaries}
Throughout the paper we will use common terminology and notations from functional analysis. The Schwartz space of functions, denoted $\mathcal{S}({\mathbb R}^n)$, is a space of infinitely differentiable functions $f: {\mathbb R}^n\rightarrow {\mathbb C}$ such that $\forall \alpha, \beta \in{\mathbb N}^n, \sup_{{\mathbf x}\in\mathbf{R}^n} $ $|{\mathbf x}^\alpha D^\beta f({\mathbf x}) |<\infty $, and equipped with a standard topology, which is complete and metrizable. A cartesian power $\mathcal{S}^m({\mathbb R}^n)$ is a set of vector-valued functions, i.e. ${\mathbf f} = (f_1,\cdots, f_m)\in \mathcal{S}^m({\mathbb R}^n)$ if and only if $f_i\in \mathcal{S}({\mathbb R}^n)$.

By the tempered distribution we understand an element from the dual space, $\mathcal{S'}({\mathbb R}^n)$. The Fourier and inverse Fourier transforms are first defined as operators $\mathcal{F}: \mathcal{S}({\mathbb R}^n)\rightarrow \mathcal{S}({\mathbb R}^n)$ by:
\vspace{-10pt}
$$
\mathcal{F}[f] (\text{\boldmath$\xi$}) = \frac{1}{\sqrt{2\pi}^n}\int_{{\mathbb R}^n} f({\mathbf x}) e^{-i\text{\boldmath$\xi$}^T{\mathbf x}} d {\mathbf x}, f\in \mathcal{S}({\mathbb R}^n) 
$$
\vspace{-8pt}
$$
\mathcal{F}^{-1}[f] ({\mathbf x}) = \frac{1}{\sqrt{2\pi}^n}\int_{{\mathbb R}^n} f(\text{\boldmath$\xi$}) e^{i\text{\boldmath$\xi$}^T{\mathbf x}} d \text{\boldmath$\xi$}, f\in \mathcal{S}({\mathbb R}^n) 
$$
and then extended to continuous bijective linear operators $\mathcal{F}, \mathcal{F}^{-1}: \mathcal{S'}({\mathbb R}^n)\rightarrow \mathcal{S'}({\mathbb R}^n)$ by the rule:
$
\mathcal{F}[\phi] (f) = \phi (\mathcal{F}[f]), \mathcal{F}^{-1}[\phi] (f) = \phi (\mathcal{F}^{-1}[f]), \phi\in \mathcal{S'}({\mathbb R}^n)
$.
The Fourier transform can be applied component-wise to objects from the cartesian power $\mathcal{S'}^m({\mathbb R}^n)$ which we will also call the tempered distributions.

If a function ${\mathbf f} = (f_1,\cdots, f_m): {\mathbb R}^n \rightarrow {\mathbb C}^m$ is such that $\int_{{\mathbb R}^n} f_i({\mathbf x})u({\mathbf x}) d {\mathbf x} <\infty$ for any $u\in \mathcal{S}({\mathbb R}^n)$ then it induces a tuple $T_{{\mathbf f}} = (T_{f_1},\cdots, T_{f_m}), T_{f_i}: \mathcal{S}({\mathbb R}^n)\rightarrow {\mathbb C}$, where $T_{f_i}(u) = \int_{{\mathbb R}^n} f_i({\mathbf x})u({\mathbf x}) d {\mathbf x}$.

For a measure $\mu$, by $L^m_{2, \mu}({\mathbb R}^n)$ we denote the Hilbert space of functions from ${\mathbb R}^n$ to ${\mathbb C}^m$, square-integrable  w.r.t $\mu$, with the inner product: $\langle {\mathbf u},{\mathbf v} \rangle = \int {\mathbf u}^\dag ({\mathbf x}){\mathbf v}({\mathbf x}) d \mu$. The induced norm is then $||{\mathbf u}||_{\mu} = \sqrt{\langle {\mathbf u},{\mathbf u} \rangle }$.
A space $L^m_{2}({\mathbb R}^n)$ (i.e. when $\mu$ is Lebesgue measure) can be embedded into $\mathcal{S'}^m({\mathbb R}^n)$, i.e. $L^m_{2}({\mathbb R}^n)\hookrightarrow \mathcal{S'}^m({\mathbb R}^n)$, where ${\mathbf f}\in L^m_{2}({\mathbb R}^n)$ corresponds to a tempered distribution $T_{\mathbf f}$. Therefore, Fourier transform can be defined on $L^m_{2}({\mathbb R}^n)$ and we will use the fact that $\mathcal{F}: L^m_{2}({\mathbb R}^n)\rightarrow L^m_{2}({\mathbb R}^n)$ is a unitary operator. 

For $\phi, \psi\in \mathcal{S}({\mathbb R}^n)$ the convolution is defined as $\phi\ast\psi ({\mathbf x}) = \int_{{\mathbb R}^n} \phi({\mathbf x}-{\mathbf y})\psi({\mathbf y}) d {\mathbf y}$. For $\psi\in \mathcal{S}({\mathbb R}^n), T\in \mathcal{S'}({\mathbb R}^n)$, the convolution is defined as a tempered distribution $\psi\ast T$ such that:
$$
\psi\ast T [\phi] = T[\tilde{\psi}\ast \phi]\,\,\, \forall \phi\in \mathcal{S}({\mathbb R}^n)
$$
where $\tilde{\psi} ({\mathbf x}) = \psi(-{\mathbf x})$ and the multiplication $\psi T$ is defined by: 
$$
(\psi T) [\phi] = T[\psi \phi]
$$
Both operations can be extended to the case when $\psi\in \mathcal{S}({\mathbb R}^n), T\in \mathcal{S'}^m({\mathbb R}^n)$ by applying them to every component of $T$.

A set of infinitely differentiable functions with a compact support in ${\mathbb R}^n$ is denoted as $C_c^\infty ({\mathbb R}^n)$. The Sobolev $s,p$-norm on $C_c^\infty ({\mathbb R}^n)$ for $s\in {\mathbb N}, p\in [1,\infty)$ is defined as 
$||f||_{s,p} = \left[\sum_{|\alpha|\leq s}|D^\alpha f|^p\right]^{1/p}$. The Sobolev space $W^{s,p}$ is a the completion of $C_c^\infty ({\mathbb R}^n)$ w.r.t. the norm $||\cdot||_{s,p}$.

For a matrix $A = \begin{bmatrix}
a_{ij}
\end{bmatrix}_{1\leq i,j\leq n}$ the Frobenius norm is $||A||_F = \sqrt{\sum_{ij}a^2_{ij}}$.
\section{Problem formulation}\label{I}
Let $p: {\mathbb R}^n\rightarrow {\mathbb R}_+$ be a probability density function such that $\sqrt{p}\in \mathcal{S}({\mathbb R}^n)$. The probability density function defines the Hilbert space $L^m_{2,p}({\mathbb R}^n)$, i.e. $L^m_{2,\mu}({\mathbb R}^n)$ where $d\mu = p d {\mathbf x}$.
We are also given a real-valued function ${\mathbf f}: {\mathbb R}^n\rightarrow {\mathbb R}^m$ from $L^m_{2,p}({\mathbb R}^n)$ which can be given in an arbitrary form, keeping in mind the case of ${\mathbf f}$ defined by a feed-forward neural network. Our goal is to approximate ${\mathbf f}$ in the following form (for $k$ fixed in advance):
$$
{\mathbf f}({\mathbf x}) \approx {\mathbf g}({\mathbf w}^T_1{\mathbf x}, \cdots, {\mathbf w}^T_k{\mathbf x})
$$
where ${\mathbf g}$ is an arbitrary function from $\mathcal{S}^m({\mathbb R}^k)$ and ${\mathbf w}_1, \cdots, {\mathbf w}_k\in {\mathbb R}^n$. 
\begin{theorem} \label{one} For ${\mathbf g}\in\mathcal{S}^m({\mathbb R}^k)$, we have
$\sqrt{p({\mathbf x})} {\mathbf g}({\mathbf w}^T_1{\mathbf x}, \cdots, {\mathbf w}^T_k{\mathbf x})  \in \mathcal{S}^m({\mathbb R}^n)$ and ${\mathbf g}({\mathbf w}^T_1{\mathbf x}, \cdots, {\mathbf w}^T_k{\mathbf x})  \in L^m_{2,p}({\mathbb R}^n)$.
\end{theorem}
\begin{proof}[Proof of theorem~\ref{one}] It is enough to prove the theorem for $m=1$.
W.l.o.g. we can assume that ${\mathbf w}_1, \cdots, {\mathbf w}_k$ are linearly independent. If they are linearly dependent and, e.g. ${\mathbf w}_k = \sum_{i=1}^{k-1}\alpha_i{\mathbf w}_i$, then we define $g'(s_1, \cdots, s_{k-1}) = g(s_1, \cdots, s_{k-1}, \sum_{i=1}^{k-1}\alpha_i s_i)$. It is easy to see that $g'\in \mathcal{S}({\mathbb R}^{k-1})$ and we reduced to the case of theorem for $k-1$.

If $s\in \mathcal{S}({\mathbb R}^n)$ and $A=\begin{bmatrix}
{\mathbf a}_1, \cdots, {\mathbf a}_n
\end{bmatrix}$ is an invertible matrix, then $s(A{\mathbf x})\in \mathcal{S}({\mathbb R}^n)$. 
Indeed, if we denote ${\mathbf y} = A{\mathbf x}$ and $A^{-1}=\begin{bmatrix}
{\mathbf b}_1, \cdots, {\mathbf b}_n
\end{bmatrix}^T$, then:
\begin{equation*}
\begin{split}
x^{\alpha_1}_1 \cdots x^{\alpha_n}_n \partial^{\beta_1}_{x_1} \cdots \partial^{\beta_n}_{x_n}  s(A{\mathbf x}) = \\ 
({\mathbf b}^T_1  {\mathbf y})^{\alpha_1} \cdots ({\mathbf b}^T_n {\mathbf y})^{\alpha_n} \cdot ({\mathbf a}^T_1   \partial_{{\mathbf y}})^{\beta_1}  \cdots ({\mathbf a}^T_n  \partial_{{\mathbf y}})^{\beta_n}  s({\mathbf y})
\end{split}
\end{equation*}
and after opening all the brackets we will obtain a finite sum of expressions of the kind ${\mathbf y}^\alpha D^\beta f$ that is bounded. In fact, we proved that Schwartz class is invariant under invertible linear change of variables.

Thus, if we complete ${\mathbf w}_1, \cdots, {\mathbf w}_k$ with ${\mathbf w}_{k+1}, \cdots, {\mathbf w}_n$ to form a basis in ${\mathbb R}^n$, and make the change of variables $y_i = {\mathbf w}^T_i{\mathbf x}$, then from $\sqrt{p({\mathbf x})} g({\mathbf w}^T_1{\mathbf x}, \cdots, {\mathbf w}^T_k{\mathbf x})$ we obtain a function $\sqrt{q({\mathbf y})} g(y_1, \cdots, y_k), \sqrt{q}\in \mathcal{S}({\mathbb R}^n)$. It remains to prove that this function is also in $\mathcal{S}({\mathbb R}^n)$.

For any $\alpha, \beta \in{\mathbb N}^n$ the expression ${\mathbf y}^\alpha D^\beta \sqrt{q({\mathbf y})}  g(y_1, \cdots, y_k)$ will be a sum if terms $({\mathbf y}^\alpha D^{\beta'}\sqrt{q({\mathbf y})} ) (D^{\beta''} g(y_1, \cdots, y_k))$ each of them being bounded.

Eventually, we note that $\mathcal{S}^m({\mathbb R}^n)\subseteq L^m_{2}({\mathbb R}^n)$ and therefore ${\mathbf g}({\mathbf w}^T_1{\mathbf x}, \cdots, {\mathbf w}^T_k{\mathbf x})  \in L^m_{2,p}({\mathbb R}^n)$.
\end{proof}
If we choose the squared error as the loss function, then we come to the following optimizational problem:
\begin{equation}\label{problem}
\begin{split}
{\mathbb E}_{{\mathbf x}\sim p({\mathbf x})} \left|{\mathbf f}({\mathbf x}) - {\mathbf g}({\mathbf w}^T_1{\mathbf x}, \cdots, {\mathbf w}^T_k{\mathbf x})\right|^2 = \\
 ||{\mathbf f}({\mathbf x}) - {\mathbf g}({\mathbf w}^T_1{\mathbf x}, \cdots, {\mathbf w}^T_k{\mathbf x})||^2_{L^m_{2,p}} \rightarrow \min_{{\mathbf g}, {\mathbf w}_1, \cdots, {\mathbf w}_k}
\end{split}
\end{equation}
The problem is non-convex and the minimum is taken over infinite-dimensional object. Let us reveal the structure of the objective:
\begin{equation*}
\begin{split}
||{\mathbf f}({\mathbf x}) - {\mathbf g}({\mathbf w}^T_1{\mathbf x}, \cdots, {\mathbf w}^T_k{\mathbf x})||_{L^m_{2,p}} = \\
||\sqrt{p({\mathbf x})}\,\,{\mathbf f}({\mathbf x}) - \sqrt{p({\mathbf x})} {\mathbf g}({\mathbf w}^T_1{\mathbf x}, \cdots, {\mathbf w}^T_k{\mathbf x})||_{L^m_{2}} \\
\end{split}
\end{equation*}
We can apply Fourier transform to our functions, taking into account that Fourier transform is unitary on $L^m_{2}({\mathbb R}^n)$.
\begin{equation*}
\begin{split}
 ||\sqrt{p({\mathbf x})}\,\,{\mathbf f}({\mathbf x}) - \sqrt{p({\mathbf x})} {\mathbf g}(\cdots)||_{L^m_{2}} = \\ 
||\mathcal{F}\left[\sqrt{p({\mathbf x})}\,\,{\mathbf f}({\mathbf x})\right] - \mathcal{F}\left[\sqrt{p({\mathbf x})} {\mathbf g}(\cdots)\right]||_{L^m_{2}}
\end{split}
\end{equation*}
Let us denote ${\mathbf f'} = \sqrt{2\pi}^n\mathcal{F}\left[\sqrt{p({\mathbf x})}\,\,{\mathbf f}({\mathbf x})\right], \gamma = \mathcal{F}\left[\sqrt{p({\mathbf x})}\right]$. 
The following statement is an application of the convolution theorem to our case: 
\begin{theorem}\label{two}
If ${\mathbf l}({\mathbf x}) = {\mathbf g}({\mathbf w}^T_1{\mathbf x}, \cdots, {\mathbf w}^T_k{\mathbf x})$ and ${\mathbf k} = \mathcal{F}\left[\sqrt{p({\mathbf x})}\,\, {\mathbf l}({\mathbf x}) \right]$, then $T_{\mathbf l}\in \mathcal{S'}^m({\mathbb R}^n)$ and 
$$T_{\mathbf k} = \frac{1}{\sqrt{2\pi}^n}\gamma\ast \mathcal{F}\left[T_{\mathbf l}\right]$$
\end{theorem}
\begin{proof}[Proof of theorem~\ref{two}] W.l.o.g. we again assume that $m=1$.
For $l({\mathbf x}) = g({\mathbf w}^T_1{\mathbf x}, \cdots, {\mathbf w}^T_k{\mathbf x})$ we have:
$$
||l||_{L_\infty({\mathbb R}^n)} \leq ||g||_{L_\infty({\mathbb R}^n)} < \infty
$$
I.e. $l\in L_\infty({\mathbb R}^n)$. Unfortunately, $l$ is not a rapidly decreasing function, because ${\mathbf w}^T_1{\mathbf x}=s_1, \cdots, {\mathbf w}^T_k{\mathbf x}=s_1$, in general, defines a nonempty affine subspace and $l$'s value on the whole subspace will be constant $g(s_1, \cdots, s_k)$. Therefore, the Fourier transform of $l$ is not necessarily an ordinary function.

Since $L_\infty({\mathbb R}^n)\hookrightarrow \mathcal{S'}({\mathbb R}^n)$,
$$
T_l[\phi] = \int_{{\mathbb R}^n} l({\mathbf x})\phi({\mathbf x}) d {\mathbf x}, \phi \in\mathcal{S}({\mathbb R}^n)
$$
is a continuous operator (i.e. a tempered distribution), therefore $\mathcal{F}[T_l]$ is also a tempered distribution. 

By definition $T_{k} = \mathcal{F}\left[\sqrt{p}T_{l}\right]$. 
Let us prove that 
$$T_{k} = \frac{1}{\sqrt{2\pi}^n}\gamma\ast \mathcal{F}\left[T_{l}\right]$$

Since $T_l\in \mathcal{S'}({\mathbb R}^n)$, there exists a sequence of functions $\phi_1, \phi_2, \cdots \in \mathcal{S}({\mathbb R}^n)$, such that 
\begin{equation*}
\begin{split}
T_{\phi_n}\rightarrow T_{l}, n\rightarrow \infty \textsc{ or } \\ 
\forall \phi\in \mathcal{S}({\mathbb R}^n),\,\,\, \int_{{\mathbb R}^n} \phi_n ({\mathbf x}) \phi ({\mathbf x}) d {\mathbf x} \rightarrow \int_{{\mathbb R}^n} l ({\mathbf x}) \phi ({\mathbf x}) d {\mathbf x}
\end{split}
\end{equation*}
The latter follows from the well-known fact that $\mathcal{S}({\mathbb R}^n)$ is dense in $\mathcal{S'}({\mathbb R}^n)$. 

It is easy to see that
\begin{equation*}
\begin{split}
\sqrt{p}T_{\phi_n}\rightarrow \sqrt{p}T_{l}, n\rightarrow \infty \textsc{ or } \forall \psi\in \mathcal{S}({\mathbb R}^n),\\ \int_{{\mathbb R}^n} \sqrt{p({\mathbf x})}\phi_n ({\mathbf x}) \psi ({\mathbf x}) d {\mathbf x} \rightarrow \int_{{\mathbb R}^n} \sqrt{p({\mathbf x})}l ({\mathbf x}) \psi ({\mathbf x}) d {\mathbf x}
\end{split}
\end{equation*}
because we can set $\phi = \sqrt{p({\mathbf x})}\psi\in \mathcal{S}({\mathbb R}^n)$ in the former expression.

The convolution theorem states that for any 2 functions $u, v\in \mathcal{S}({\mathbb R}^n)$ we have:
$$
\mathcal{F}\left[uv\right] = \frac{1}{\sqrt{2\pi}^n}\mathcal{F}\left[u\right] \ast \mathcal{F}\left[v\right],\,\,\mathcal{F}\left[uT_v\right] = \frac{1}{\sqrt{2\pi}^n}\mathcal{F}\left[u\right] \ast \mathcal{F}\left[T_v\right]
$$
Therefore:
$$
\mathcal{F}\left[\sqrt{p}T_{\phi_n}\right] = \frac{1}{\sqrt{2\pi}^n}\gamma\ast \mathcal{F}\left[T_{\phi_n}\right]
$$
Since $\mathcal{F}: \mathcal{S'}({\mathbb R}^n)\rightarrow \mathcal{S'}({\mathbb R}^n)$ is a continuous operator, then $\mathcal{F}\left[T_{\phi_n}\right]\rightarrow \mathcal{F}\left[T_{l}\right]$ and $\mathcal{F}\left[\sqrt{p}T_{\phi_n}\right]\rightarrow \mathcal{F}\left[\sqrt{p}T_{l}\right]$ in $\mathcal{S'}({\mathbb R}^n)$. In order to obtain the needed result it remains to show that the convolution operator $C_\gamma: \mathcal{S'}({\mathbb R}^n)\rightarrow \mathcal{S'}({\mathbb R}^n)$, $C_\gamma(T) = \gamma\ast T$
is also continuous.

By definition $\gamma\ast T [\phi] = T[{\tilde \gamma}\ast\phi]$ where ${\tilde \gamma}({\mathbf x}) = \gamma(-{\mathbf x})$.
I.e. we have to show that if 
$$
T_i\rightarrow T\textsc{ or }\forall \phi\in \mathcal{S}({\mathbb R}^n), \,\, T_i[\phi]\rightarrow T[\phi]
$$
then
$$
\gamma\ast T_i\rightarrow \gamma\ast T \textsc{ or }\forall \psi\in \mathcal{S}({\mathbb R}^n), \,\, T_i[{\tilde \gamma}\ast\psi]\rightarrow T[{\tilde \gamma}\ast\psi]
$$
The latter is obvious if we can set $\phi = {\tilde \gamma}\ast\psi\in \mathcal{S}({\mathbb R}^n)$ in the former expression.
Thus, theorem proved.

\end{proof}

The basic phenomenon behind our approach to optimization of~\eqref{problem} is the following statement:
\begin{theorem}\label{basic}
A function ${\mathbf l}({\mathbf x})$ can be represented as ${\mathbf l}({\mathbf x}) = {\mathbf g}({\mathbf w}^T_1{\mathbf x}, \cdots, {\mathbf w}^T_k{\mathbf x}), {\mathbf g}\in \mathcal{S}^m({\mathbb R}^k)$ if and only if there is an orthonormal basis $\{{\mathbf a}_1, \cdots, {\mathbf a}_n\}\subseteq {\mathbb R}^n$ such that:
\begin{equation}\label{gen}
\mathcal{F}[T_{{\mathbf l}}] = {\mathbf r}({\mathbf a}^T_1{\mathbf x}, \cdots, {\mathbf a}^T_{k'}{\mathbf x}) \prod_{i=k'+1}^n \delta({\mathbf a}^T_i{\mathbf x}), {\mathbf r} \in \mathcal{S}^m ({\mathbb R}^k)
\end{equation}
where $\delta(\cdot)$ -- Dirac's delta function.
Moreover, $span({\mathbf a}_1, \cdots, {\mathbf a}_{k'}) = span({\mathbf w}_1, \cdots, $ ${\mathbf w}_k)$.
\end{theorem}

\begin{proof}[Sketch of the proof of theorem~\ref{basic}]
W.l.o.g. we can assume that $m=1$ and ${\mathbf w}_1, \cdots, {\mathbf w}_k$ are linearly independent.
A rigorous proof of the theorem would require a carefull checking of certain integral identitites. Instead we will present a sketch of the proof at the level of strictness common to theoretical physics papers.

($\Rightarrow$) We also can assume that ${\mathbf w}_1, \cdots, {\mathbf w}_k$ are orthonormal. Indeed, after every redefinition of $g$ given by the rule $g(s_1, \cdots, s_{k}) \leftarrow g(s_1, \cdots, s_{i}+\alpha s_j, \cdots, s_k)$ we get the same function $l$ if we simultaneously transform ${\mathbf w}_i$ to ${\mathbf w}_{i}-\alpha {\mathbf w}_j$. By making such redefinitions, we can always orthogonolize ${\mathbf w}_1, \cdots, {\mathbf w}_k$ by Gramm-Schmidt process with a subsequent scaling of $g$'s arguments.

Let us complete ${\mathbf w}_1, \cdots, {\mathbf w}_k$ with ${\mathbf w}_{k+1}, \cdots, {\mathbf w}_n$ to form an orthonormal basis in ${\mathbb R}^n$ and set:
$$
Q = \begin{bmatrix}
{\mathbf w}_1, \cdots , {\mathbf w}_n
\end{bmatrix} = \begin{bmatrix}
Q_1 , Q_2
\end{bmatrix}, Q_1\in {\mathbb R}^{n\times k}, Q_2\in {\mathbb R}^{n\times (n-k)}
$$
Then in the Fourier transform formula we will make the change of variables ${\mathbf x} = Q \begin{bmatrix}
{\mathbf y}_1 \\ {\mathbf y}_2
\end{bmatrix} = 
Q_1 {\mathbf y}_1 + Q_2 {\mathbf y}_2$, ${\mathbf y}_1\in {\mathbb R}^k$, ${\mathbf y}_2\in {\mathbb R}^{n-k}$:
\begin{equation*}
\begin{split}
\mathcal{F}[l](\text{\boldmath$\xi$}) = \frac{1}{\sqrt{2\pi}^n}\int_{{\mathbb R}^n} g({\mathbf w}^T_1{\mathbf x}, \cdots, {\mathbf w}^T_k{\mathbf x}) e^{-i\text{\boldmath$\xi$}^T{\mathbf x}} d {\mathbf x} = \\
 \frac{1}{\sqrt{2\pi}^n}\int_{{\mathbb R}^n} g({\mathbf y}_1) e^{-i\text{\boldmath$\xi$}^T Q \begin{bmatrix}
{\mathbf y}_1 \\ {\mathbf y}_2
\end{bmatrix}} d {\mathbf y}_1 d {\mathbf y}_2 = 
\\
=\frac{1}{\sqrt{2\pi}^n}\int_{{\mathbb R}^n} g({\mathbf y}_1) e^{-i(Q_1^T\text{\boldmath$\xi$})^T 
{\mathbf y}_1 -i(Q_2^T\text{\boldmath$\xi$})^T{\mathbf y}_2
} d {\mathbf y}_1 d {\mathbf y}_2 = \\
\frac{1}{\sqrt{2\pi}^n}\int_{{\mathbb R}^{k}} g({\mathbf y}_1) e^{-i(Q_1^T\text{\boldmath$\xi$})^T 
{\mathbf y}_1
} d {\mathbf y}_1 \cdot
\\
\cdot \int_{{\mathbb R}^{n-k}} e^{-i(Q_2^T\text{\boldmath$\xi$})^T{\mathbf y}_2
}d {\mathbf y}_2 = \sqrt{2\pi}^{n-k} \mathcal{F}[g](Q_1^T\text{\boldmath$\xi$}) \delta^{n-k} (Q_2^T\text{\boldmath$\xi$})
\end{split}
\end{equation*}
where $\delta^{n-k}(s_1, \cdots, s_{n-k}) = \prod_{i=1}^{n-k} \delta(s_i)$. Here we used that $\int_{{\mathbb R}^{n-k}} e^{-i{\mathbf z}^T{\mathbf y}_2
}d {\mathbf y}_2 = (2\pi)^{n-k}\delta^{n-k}({\mathbf z})$. Thus, we obtain the needed representation.

($\Leftarrow$) Suppose that:
$$\mathcal{F}[l] = r({\mathbf a}^T_1{\mathbf x}, \cdots, {\mathbf a}^T_{k'}{\mathbf x}) \prod_{i=k'+1}^n \delta({\mathbf a}^T_i{\mathbf x})$$
Using inverse Fourier transform we get:
\begin{equation*}
\begin{split}
l(\text{\boldmath$\xi$}) = \mathcal{F}^{-1}\left[\mathcal{F}[l]\right](\text{\boldmath$\xi$}) =\\
\frac{1}{\sqrt{2\pi}^n} \int_{{\mathbb R}^n} r({\mathbf a}^T_1{\mathbf x}, \cdots, {\mathbf a}^T_{k'}{\mathbf x}) \prod_{i=k'+1}^n \delta({\mathbf a}^T_i{\mathbf x}) e^{i{\mathbf x}^T\text{\boldmath$\xi$}} d {\mathbf x}
\end{split}
\end{equation*}
After the change of variables ${\mathbf x} = O{\mathbf y}$, where $$O = \begin{bmatrix}
{\mathbf a}_1, \cdots, {\mathbf a}_n
\end{bmatrix}$$
we get: 
\begin{equation*}
\begin{split}
l(\text{\boldmath$\xi$}) = \frac{1}{\sqrt{2\pi}^n} \int_{{\mathbb R}^n} r(y_{1:k'}) \prod_{i=k'+1}^n \delta(y_i) e^{i\sum_{i=1}^n y_i {\mathbf a}^T_i\text{\boldmath$\xi$}}  dy_{1:n} =  \\
\frac{1}{\sqrt{2\pi}^n} \int_{{\mathbb R}^n} r(y_{1:k'})  e^{i\sum_{i=1}^{k'} y_i {\mathbf a}^T_i\text{\boldmath$\xi$}}  dy_{1:k'} = \\ \frac{1}{\sqrt{2\pi}^{n-k}} \tilde{g} ({\mathbf a}^T_1\text{\boldmath$\xi$}, \cdots, {\mathbf a}^T_k\text{\boldmath$\xi$})
\end{split}
\end{equation*}
where $\tilde{g} =  \mathcal{F}^{-1}[r]$.
\end{proof}

Substantively, the theorem claims that if the function's value depends only on the projection of an argument ${\mathbf x}$ on $span({\mathbf w}_1, \cdots, {\mathbf w}_k)$, then frequencies from the spectrum of such function are all in $span({\mathbf w}_1, \cdots, {\mathbf w}_k)$.

\begin{definition} A set of tempered distributions of the form~\eqref{gen} is denoted as $\mathcal{G}_k$ and called a set of functions with $k$-dimensional support.
\end{definition}
Thus, our problem becomes equivalent to:
$$
||{\mathbf f'} - {\mathbf k}||_{L^m_{2}} \rightarrow \min_{T_{\mathbf k} = \gamma\ast {\mathbf g}, {\mathbf g}\in\mathcal{G}_k}
$$
For simplicity of our notation, let us use ${\mathbf k}$ and $T_{\mathbf k}$ interchangeably (from the context it is always clear what we mean). Thus, our problem is:
\begin{equation}\label{fourier-one}
||{\mathbf f'} - \gamma\ast {\mathbf g}||_{L^m_{2}} \rightarrow \min_{{\mathbf g}\in\mathcal{G}_k}
\end{equation}

Note that if we would restrict ${\mathbf g}$ to be any ordinary function, the latter problem is known in the theory of inverse problems. E.g., in a case when $\gamma({\mathbf x}) = e^{-|{\mathbf x}|^2/2}$, a problem of finding $g$ such that $f' = \gamma\ast g$ is known as {\em the deconvolution of gaussian kernel}, and has many applications in mathematical physics~\cite{Saitoh}, \cite{Saitoh2}, \cite{Ulmer}. But with our type of restriction, besides that we cannot guarantee that the minimum is attainable on a function from $\mathcal{G}_k$, the set $\mathcal{G}_k$ itself does not suit as a good optimization space as it lacks obvious metrics, completeness properties etc. 

Instead of minimization over tempered distributions we will relax the property that the support of the function $g$ is strictly $k$-dimensional, reducing the problem to optimization over ordinary functions:
$$
||{\mathbf f'} - \gamma\ast {\mathbf g}||_{L^m_{2}} \rightarrow \min_{{\mathbf g}: R({\mathbf g})\leq \epsilon}
$$
where $R({\mathbf g})$ is a penalty term that penalizes ${\mathbf g}$ if ``the dimensionality of its support is greater than $k$''. In the next section we describe one natural approach to construct such a penalty term $R$.

\section{Penalty function}\label{II}
Let $I: {\mathbb C}^m \rightarrow {\mathbb R}_{+} = \{x\in {\mathbb R}| x\geq 0\}$ be a continuous function such that $I({\mathbf 0}) = 0$ and $I({\mathbf c})\ne 0, {\mathbf c}\ne {\mathbf 0}$. Let us consider a set of functions:
$$
L_I = \left\{{\mathbf g}:{\mathbb R}^n\rightarrow {\mathbb C}^m | \int_{{\mathbb R}^n}I({\mathbf g}({\mathbf x})) d {\mathbf x} < \infty\right\}
$$
We believe that practically the most interesting case is $I({\mathbf x}) = |{\mathbf x}|^\alpha, \alpha > 0$. Since $I({\mathbf g}({\mathbf x}))\geq 0$, we will correspond to ${\mathbf g}\in L_I$ the finite measure function (induced by the density $I({\mathbf g}({\mathbf x}))$): 
$$
\mu_{\mathbf g} (A) = \int_{A}I({\mathbf g}({\mathbf x})) d {\mathbf x}, A\subseteq {\mathbb R}^n, \,\,\,\mu_{\mathbf g} ({\mathbb R}^n) < \infty
$$
on the $\sigma$-algebra of Lebesgue measurable sets.
Any finite measure $\mu$ induces the probability measure $\mu^P$ via the normalization: $
\mu^P (A) = \frac{\mu (A)}{\mu ({\mathbb R}^n)}
$. We will call a finite measure $\mu$ on ${\mathbb R}^n$ {\em a $k$-dimensional measure} if there is a $k$-dimensional linear subspace $S\subseteq {\mathbb R}^n$ such that $\mu^P(S) = 1$.

In the previous section we proved that our problem~\eqref{problem} can be reduced to optimization task~\eqref{fourier-one} over functions with $k$-dimensional support. As we have already pointed out, $\mathcal{G}_k$ (as well as $\mathcal{S'}^m ({\mathbb R}^n)$) lacks standard metrics on it, so we need to devise a certain way to measure a distance from an ordinary function ${\mathbf g}$ to a set $\mathcal{G}_k$.
If ${\mathbf g}$ is an ordinary function, then its support cannot be strictly $k$-dimensional. It is natural to define a distance till $\mathcal{G}_k$ as $\min_{\mu \textsc{\tiny{ is k-dimensional}}}\rho(\mu_{\mathbf g}, \mu)$, for a proper distance function $\rho$ on measures. It turns out that $k$-dimensional measures can be characterized in a very simple way:

\begin{theorem} Let $\mu$ be a finite measure on ${\mathbb R}^n$ such that $\forall i,j, k, l\,\,\int_{{\mathbb R}^n} x_i x_j x_k x_l d \mu < \infty$. The measure $\mu$ is $k$-dimensional if and only if 
$$
rank(\mathcal{M}) \leq k
$$
where $
\mathcal{M} = \int_{{\mathbb R}^n} {\mathbf x} {\mathbf x}^T d \mu
$.
\end{theorem}
\begin{proof} 
Let ${\mathbf x}_1,...,{\mathbf x}_N$ be i.i.d. random vectors sampled according to $\mu^P$ and ${\mathbf x}_i = \begin{bmatrix}
x_{i1}, \cdots, x_{in}
\end{bmatrix}^T$. A natural estimator for the matrix of second moments $\begin{bmatrix}
{\mathbb E}_{\text{\boldmath$\xi$}\sim \mu^P} [\xi_i \xi_j]
\end{bmatrix}_{1\leq i,j\leq n}$ is:
$$
\frac{1}{N}\sum_{i=1}^N {\mathbf x}_i{\mathbf x}_i^T = \frac{1}{N}X^TX
$$
where $X = \begin{bmatrix}
{\mathbf x}_1,...,{\mathbf x}_N
\end{bmatrix}^T$.

This estimator is consistent, i.e.:
$$
\lim_{N\rightarrow \infty}P\left[||\frac{1}{N}X^TX- \mathcal{M}||_F  >\epsilon\right] = 0, 
$$

If we denote $\frac{1}{N}X^TX = \begin{bmatrix}
s_{ij}
\end{bmatrix}_{1\leq i,j\leq n}$, then the latter can be shown after analysis of:
$
s_{ij} = \frac{1}{N}\sum_{k=1}^N x_{ki}x_{kj}
$.
Indeed, $\{x_{ki}x_{kj}\}_{k=1}^N$ are i.i.d. random variables with finite second moment $\int_{{\mathbb R}^n} x^2_{ki}x^2_{kj} d \mu $. Therefore, by weak law of large numbers:
$$
\lim_{N\rightarrow \infty}P\left[|\frac{1}{N}\sum_{k=1}^N x_{ki}x_{kj}-{\mathbb E}_{\text{\boldmath$\xi$}\sim \mu^P} [\xi_i \xi_j] |  >\epsilon\right] = 0, 
$$
I.e
$
\lim_{N\rightarrow \infty}P\left[|s_{ij}-{\mathbb E}_{\text{\boldmath$\xi$}\sim \mu^P} [\xi_i\xi_j] |  >\epsilon\right] = 0, 
$
and therefore:
$$
\lim_{N\rightarrow \infty}P\left[||\frac{1}{N}X^TX- \mathcal{M}||_F  >\epsilon\right] = 0. 
$$

($\Rightarrow$) Now suppose that $rank(\mathcal{M})\leq k$. I.e. we can find orthonormal vectors ${\mathbf v}_1, \cdots, {\mathbf v}_{n-k}$ such that $\mathcal{M} {\mathbf v}_i = {\mathbf 0}$. Since $ \frac{1}{N}||X{\mathbf v}_i||^2 = \frac{1}{N}{\mathbf v}^T_iX^TX{\mathbf v}_i \leq ||\frac{1}{N}X^TX{\mathbf v}_i|| = ||(\frac{1}{N}X^TX- \mathcal{M}){\mathbf v}_i|| \leq ||\frac{1}{N}X^TX- \mathcal{M}||_F ||{\mathbf v}_i||$, then $P\left[\frac{1}{N}||X{\mathbf v}_i||^2 >\epsilon\right] \leq P\left[||\frac{1}{N}X^TX- \mathcal{M}|| >\epsilon\right]$ and:
\begin{equation}\label{lowdim}
\lim_{N\rightarrow \infty}P\left[\frac{1}{N}||X{\mathbf v}_i||^2 >\epsilon\right] =0
\end{equation}

Let us now introduce a random variable $Z = ({\mathbf v}^T_i \text{\boldmath$\xi$})^2$. It is easy to see that a natural estimator of ${\mathbb E}_{\text{\boldmath$\xi$}\sim \mu^P} Z$ is the following expression:
\begin{equation}
\frac{1}{N}||X{\mathbf v}_i||^2 = \frac{1}{N}\sum_{i=1}^N ({\mathbf v}^T_i{\mathbf x}_i)^2 
\end{equation}
Consistency of that estimator, i.e. the statement that
$$
\lim_{N\rightarrow \infty}P\left[|\frac{1}{N}||X{\mathbf v}_i||^2-{\mathbb E}_{\text{\boldmath$\xi$}\sim \mu^P} Z| >\epsilon\right] =0
$$
also follows from the weak law of large numbers, due to ${\mathbb E}_{\text{\boldmath$\xi$}\sim \mu^P} Z^2<\infty$. This, together with~\eqref{lowdim} implies that ${\mathbb E}_{\text{\boldmath$\xi$}\sim \mu^P} Z = 0$. I.e. ${\mathbf v}^T_i \text{\boldmath$\xi$} = {\mathbf 0}$ with probability 1. The latter means that 
$$
P\left[\cap_{i=1}^{n-k} \{\text{\boldmath$\xi$} | {\mathbf v}^T_i \text{\boldmath$\xi$} = {\mathbf 0}\}\right] = 1
$$
and $\mu$ is $k$-dimensional.

($\Leftarrow$) If $\mu$ is $k$-dimensional, then there is a $k$-dimensional linear subspace $S\subseteq {\mathbb R}^n$ such that $\mu^P(S) = 1$.  Let $\{{\mathbf v}_i\}_{i=1}^n$ be an orthonormal basis in ${\mathbb R}^n$ such that ${\mathbf v}_i \perp S, i>k$. Then:
$$
\mathcal{M} = \int_{{\mathbb R}^n} {\mathbf x} {\mathbf x}^T d \mu = \int_{{\mathbb R}^n} \sum_{i=1}^n{\mathbf v}_i {\mathbf v}^T_i {\mathbf x} {\mathbf x}^T \sum_{i=1}^n{\mathbf v}_i {\mathbf v}^T_i d\mu
$$
Since $\int_{{\mathbb R}^n}({\mathbf v}^T_i {\mathbf x})^2 d\mu = 0, i>k$, then:
\begin{equation*}
\begin{split}
\mathcal{M} = \int_{{\mathbb R}^n} \sum_{i=1}^k{\mathbf v}_i {\mathbf v}^T_i {\mathbf x} {\mathbf x}^T \sum_{i=1}^n{\mathbf v}_i {\mathbf v}^T_i d\mu = \\
\sum_{i=1}^k{\mathbf v}_i {\mathbf v}^T_i\int_{{\mathbb R}^n}  {\mathbf x} {\mathbf x}^T d\mu \sum_{i=1}^k{\mathbf v}_i {\mathbf v}^T_i
\end{split}
\end{equation*}
and we see that $rank(\mathcal{M}) \leq k$.
\end{proof}

Let us now define $L_{I,2} = \left\{{\mathbf g}\in L_I |\int_{{\mathbb R}^n} {\mathbf x} {\mathbf x}^T d \mu_{\mathbf g} < \infty \right\}$ and for any ${\mathbf g}\in L_{I,2}$ introduce $\mathcal{M}_{\mathbf g} = \int_{{\mathbb R}^n} {\mathbf x} {\mathbf x}^T d \mu_{\mathbf g} = \int_{{\mathbb R}^n} {\mathbf x} {\mathbf x}^T I({\mathbf g}({\mathbf x})) d {\mathbf x}$. Note that $\mathcal{M}_{\mathbf g}$ is a positive semidefinite matrix, and therefore, the square root $\mathcal{M}^{1/2}_{\mathbf g}$ is defined. Our definition for a penalty function $R: L_{I,2}\rightarrow {\mathbb R}$ is:
\begin{equation}
R({\mathbf g}) = \min_{\mathcal{M}\in {\mathbb R}^{n\times n}: rank(\mathcal{M}) \leq k}||\mathcal{M}^{1/2}_{\mathbf g}-\mathcal{M}||^2_{F}
\end{equation}
It is natural to expect that if $R({\mathbf g})\leq \epsilon$ where $\epsilon > 0$ is small, i.e. if $\mathcal{M}^{1/2}_{\mathbf g}$ (together with $\mathcal{M}_{\mathbf g}$) is close to some rank $k$ matrix, then the support of ${\mathbf g}$ is approximable with a $k$-dimensional linear subspace. I.e. our goal is to develop an algorithm for the following problem:
\begin{equation} \label{fourier-two}
||{\mathbf f}' - \gamma\ast {\mathbf g}||_{L^m_{2}} \rightarrow \min_{{\mathbf g}\in L_{I,2}: R({\mathbf g})\leq \epsilon}
\end{equation}
\subsection{Another description of the penalty}
Let us now give an alternative description of the penalty $R({\mathbf g})$ that would suit better to the tasks of theoretical analysis of the problem~\eqref{fourier-two}.

Let $J: {\mathbb C}^m \rightarrow {\mathbb C}^l$ be a continuous function such that $J({\mathbf c})^\dag J({\mathbf c}) = I({\mathbf c})$. For example, $I({\mathbf c}) = |{\mathbf c}|^{2}$, $l=m, J({\mathbf c}) = {\mathbf c}$. By $L^{n\times l}_2({\mathbb R}^n)$ we denote a space of matrices $\begin{bmatrix}
b_{ij}({\mathbf x})\end{bmatrix}_{1\leq i \leq n, 1\leq j \leq l}$, where $b_{ij}\in L_2({\mathbb R}^n)$. 

It is easy to see that any $A\in L^{n\times l}_2({\mathbb R}^n)$ defines a bounded linear operator $O_A$ from $L^l_2({\mathbb R}^n)$ to ${\mathbb C}^n$ by the following rule:
$$
\phi\in L^l_2({\mathbb R}^n) \rightarrow^{O_A} \int_{{\mathbb R}^n} O ({\mathbf x}) \phi({\mathbf x}) d {\mathbf x}
$$
Moreover, it easy to see that all bounded linear operators from $L^l_2({\mathbb R}^n)$ to ${\mathbb C}^n$ can be given in this way.
$L^{n\times l}_2({\mathbb R}^n)$ is a Hilbert space, where the inner product is defined as:
$$
\langle A_1, A_2 \rangle_{L^{n\times l}_2({\mathbb R}^n)} = \int_{{\mathbb R}^n} Tr \left( A_1({\mathbf x})^\dag A_2 ({\mathbf x}) \right) d {\mathbf x}
$$
Recall that, for a bounded linear operator $O: \mathcal{H}_1\rightarrow \mathcal{H}_2$ between Hilbert spaces $\mathcal{H}_1, \mathcal{H}_2$, the rank of $O$ is defined as $\dim \Ima (O)$, where $\Ima (O) = \{O[\phi] \big{|} \phi\in \mathcal{H}_1\}$.

Let us define $S_{\mathbf g} = {\mathbf x}  J({\mathbf g}({\mathbf x}))^T$. 
\begin{theorem}\label{bounded}
If $\mathcal{M}_{\mathbf g} < \infty$, then $S_{\mathbf g}\in L^{n\times l}_2({\mathbb R}^n)$, $O_{S_{\mathbf g}} O_{S_{\mathbf g}}^\dag = \mathcal{M}_{\mathbf g}$, and therefore, $rank(O_{S_{\mathbf g}}) = rank(\mathcal{M}_{\mathbf g})$.
\end{theorem}
\begin{proof} [Proof of theorem~\ref{bounded}]
The fact that $S_{\mathbf g}\in L^{n\times l}_2({\mathbb R}^n)$ follows from:
\begin{equation*}
\begin{split}
||S_{\mathbf g}||^2_{L^{n\times l}_2({\mathbb R}^n)} = 
\int_{{\mathbb R}^n} |{\mathbf x}|^2 I({\mathbf g}({\mathbf x})) d {\mathbf x} 
= Tr(\mathcal{M}_{\mathbf g}) < \infty
\end{split}
\end{equation*}
The dual to $O_{S_{\mathbf g}}$ is, by definition, an operator $O_{S_{\mathbf g}}^\dag: {\mathbb C}^n\rightarrow L^l_2({\mathbb R}^n)$ that satisfies for any ${\mathbf u}\in {\mathbb C}^n, \phi\in L^l_2({\mathbb R}^n)$:
$$
\int_{{\mathbb R}^n} {\mathbf u}^\dag {\mathbf x} J({\mathbf g}({\mathbf x}))^T \phi({\mathbf x}) d {\mathbf x} = \int_{{\mathbb R}^n} O_{S_{\mathbf g}}^\dag[{\mathbf u}] ({\mathbf x})^\dag \phi({\mathbf x}) d {\mathbf x}
$$
It is easy to see that $O_{S_{\mathbf g}}^\dag[{\mathbf u}]({\mathbf x}) = {\mathbf x}^T {\mathbf u}  J({\mathbf g}({\mathbf x}))^\ast$. Thus, $O_{S_{\mathbf g}} O_{S_{\mathbf g}}^\dag : {\mathbb C}^n\rightarrow {\mathbb C}^n$ acts on ${\mathbf u}\in {\mathbb C}^n$ as:
$$
{\mathbf u} \rightarrow^{O_{S_{\mathbf g}}^\dag} \hspace{-3pt}{\mathbf x}^T {\mathbf u}  J({\mathbf g}({\mathbf x}))^\ast\rightarrow^{O_{S_{\mathbf g}}}\hspace{-3pt} \int_{{\mathbb R}^n} {\mathbf x} J({\mathbf g}({\mathbf x}))^T {\mathbf x}^T {\mathbf u}  J({\mathbf g}({\mathbf x}))^\ast d {\mathbf x}
$$
The latter is equal to $\mathcal{M}_{\mathbf g} {\mathbf u} $ and we conclude that $O_{S_{\mathbf g}} O_{S_{\mathbf g}}^\dag = \mathcal{M}_{\mathbf g}$ and $rank(O_{S_{\mathbf g}}) = rank(\mathcal{M}_{\mathbf g})$.
\end{proof}

Eckart-Young theorem from the theory of Singular Value Decomposition (SVD) gives us that 
$$
R({\mathbf g}) = \min_{\mathcal{M}\in {\mathbb R}^{n\times n}: rank(\mathcal{M}) \leq k}||\mathcal{M}^{1/2}_{\mathbf g}-\mathcal{M}||^2_{F} = \sum_{i=k+1}^{n} \lambda_i
$$
where $\lambda_1 \geq \cdots \geq \lambda_{n}>0$ are eigenvalues of $\mathcal{M}_{\mathbf g}\hspace{-2pt}=\hspace{-2pt}\mathcal{M}^{\frac{1}{2}T}_{\mathbf g}\hspace{-2pt}\mathcal{M}^{\frac{1}{2}}_{\mathbf g}$.
Due to the relationship $O_{S_{\mathbf g}} O_{S_{\mathbf g}}^\dag = \mathcal{M}_{\mathbf g}$ the following becomes true:
\begin{theorem}\label{represent}
$
R({\mathbf g}) = \min_{S\in L^{n\times l}_2({\mathbb R}^n): rank(O_{S}) \leq k}||S_{\mathbf g}-S||^2_{L^{n\times l}_2({\mathbb R}^n)}
$
\end{theorem}
We will omit a proof of that theorem because it is just a carefull checking that all arguments of Eckart-Young theorem for matrices maintain in the case of bounded linear operators from $L^l_2({\mathbb R}^n)$ to ${\mathbb C}^n$. Indeed, all arguments survive, because such operators  can have only a finite spectrum, due to the fact that ${\mathbb C}^n$ is finite-dimensional. Let us only describe an optimal $S$ on which $\min_{S\in L^{n\times l}_2({\mathbb R}^n): rank(O_S) \leq k}||S_{\mathbf g}-S||^2_{L^{n\times l}_2({\mathbb R}^n)}$ is attained.

Let ${\mathbf u}_1, \cdots {\mathbf u}_n$ be orthonormal eigenvectors of $\mathcal{M}_{\mathbf g} = O_{S_{\mathbf g}} O_{S_{\mathbf g}}^\dag$ and $\lambda_1 \geq \cdots \geq \lambda_{n}>0$ be corresponding eigenvalues. For $\sigma_i = \sqrt{\lambda_i}$ let us define
$
{\mathbf v}_i = \frac{O_{S_{\mathbf g}}^\dag[{\mathbf u}_i]}{\sigma_i}
$.
A vector ${\mathbf v}_i$ corresponds to a function:
$$
{\mathbf v}_i ({\mathbf x}) = \frac{{\mathbf x}^T {\mathbf u}_i  J({\mathbf g}({\mathbf x}))^\ast}{\sigma_i} \in L^l_2({\mathbb R}^n)
$$
It is easy to see that ${\mathbf v}_1, \cdots {\mathbf v}_n$ is an orthonormal basis in $\Ima  O_{S_{\mathbf g}}^\dag$, and 
SVD for $S_{\mathbf g}$ is:
$$
S_{\mathbf g} = \sum_{i=1}^n \sigma_i  {\mathbf u}_i {\mathbf v}_i^\dag = \sum_{i=1}^n {\mathbf u}_i {\mathbf u}_i^\dag {\mathbf x}  J({\mathbf g}({\mathbf x}))^T
$$
An optimal $S$ is defined by a truncation of SVD for $S_{\mathbf g}$ at $k$th term, i.e.:
\begin{equation}\label{optS}
S = \sum_{i=1}^k {\mathbf u}_i {\mathbf u}_i^\dag {\mathbf x}  J({\mathbf g}({\mathbf x}))^T = P_{\mathbf g} {\mathbf x} J({\mathbf g}({\mathbf x}))^T
\end{equation}
where $P_{\mathbf g} = \sum_{i=1}^k {\mathbf u}_i {\mathbf u}_i^\dag$ is a projection operator to first $k$ principal components of $\mathcal{M}_{\mathbf g}$. 

\section{An algorithm for $p({\mathbf x}) \propto e^{-|{\mathbf x}|^2}$ and $I({\mathbf c}) = |{\mathbf c}|^2$}\label{III}
Practically a very important probability distribution on ${\mathbb R}^n$ is the multi-variate normal distribution, i.e. $p({\mathbf x}) = \frac{1}{\sqrt{(2\pi)^n|\Sigma|}}e^{-\frac{1}{2}({\mathbf x}-\text{\boldmath$\mu$})^T\Sigma^{-1}({\mathbf x}-\text{\boldmath$\mu$})}$. For that distribution, the problem~\eqref{problem}, after an affine change of variables can be reduced to the case $p({\mathbf x}) \propto e^{-|{\mathbf x}|^2}$.

To simplify our notation, we will assume that $p({\mathbf x}) = e^{-|{\mathbf x}|^2}$ (we can hide the normalization constant inside ${\mathbf f}, {\mathbf g}$ in the objective~\eqref{problem}). Therefore, $\gamma({\mathbf x}) = \mathcal{F}[\sqrt{p({\mathbf x})}] = e^{-|{\mathbf x}|^2/2}$.

Thus, the objective of our problems \eqref{fourier-one} and~\eqref{fourier-two} is the same:
\begin{equation*}
\begin{split}
||{\mathbf f'} - \gamma\ast {\mathbf g}||^2_{L^m_{2}} = \int_{{\mathbb R}^n} \left|{\mathbf f'}({\mathbf x}) - \int_{{\mathbb R}^n} e^{-|{\mathbf x}-{\mathbf y}|^2/2}{\mathbf g}({\mathbf y}) d {\mathbf y}\right|^2 d{\mathbf x} \\ 
= ||{\mathbf f'} - \mathcal{W} [{\mathbf g}]||^2_{L^m_{2}}
\end{split}
\end{equation*}
where $\mathcal{W} [{\mathbf g}] ({\mathbf x}) \triangleq \gamma\ast {\mathbf g} = \int_{{\mathbb R}^n} e^{-|{\mathbf x}-{\mathbf y}|^2/2}{\mathbf g}({\mathbf y}) d {\mathbf y}$ is a well-known integral transform which is called {\em Weierstrass transform}. Recall that ${\mathbf f'} = \sqrt{2\pi}^n\mathcal{F}[\sqrt{p({\mathbf x})}{\mathbf f}] = \gamma\ast {\mathbf {\hat f}} = \mathcal{W} [{\mathbf {\hat f}}]$ where ${\mathbf {\hat f}} = \mathcal{F}[{\mathbf f}]$.

The difference between problems is that in~\eqref{fourier-one} we optimize over tempered distributions ${\mathbf g}\in \mathcal{G}_k$ and in~\eqref{fourier-two} we optimize over $\Omega_\epsilon = \left\{{\mathbf g}\in L_{I,2}| R({\mathbf g})\leq \epsilon\right\}$. Together with $p({\mathbf x}) = e^{-|{\mathbf x}|^2}$ we will assume that $I({\mathbf c}) = |{\mathbf c}|^2$. It is easy to see that in our case: 
$$L_{I,2} = \{{\mathbf g}\in L^m_2({\mathbb R}^n)| \int_{{\mathbb R}^n} |{\mathbf x}|^2|{\mathbf g}({\mathbf x})|^2 d{\mathbf x} < \infty\}$$
A well-known characterization of Sobolev spaces in terms of Fourier transform~\cite{Joran} states that:
$$
W^{1,2} = \left\{f\in L_2({\mathbb R}^n)| (1+|\text{\boldmath$\xi$}|^2)^{1/2}\mathcal{F}[f]\in L_2({\mathbb R}^n)\right\}
$$
The latter implies that $L_{I,2} = \{{\mathbf g} = (g_1, \cdots, g_m) | \mathcal{F}^{-1}[g_i]\in W^{1,2}\}$, i.e. $L_{I,2}$ is just an image of $(W^{2,1})^m$ under $\mathcal{F}$. This fact will play its role in the next section.

Let us describe a natural heuristics for our problem when $I({\mathbf c}) = |{\mathbf c}|^2$. Given that $I$, we define $J({\mathbf c}) = {\mathbf c}$ and check that $I({\mathbf c}) = J({\mathbf c})^\dag J({\mathbf c})$.  
The formulation~\eqref{fourier-two} is not equivalent but connected with the following optimizational problem:
$$
||\mathcal{W} [{\mathbf {\hat f}}] - \mathcal{W} [{\mathbf g}]||^2_{L^m_{2}} +\lambda R({\mathbf g})\rightarrow \min_{{\mathbf g}\in L_{I,2}}
$$
By varying the parameter $\lambda$ we control a contribution of the penalty $R$. Thus, it is natural to expect that an increase of $\lambda$ will force an optimal ${\mathbf g}$ to be $k$-dimensional.
Taking into account the representation of $R$ given in theorem~\ref{represent} we can rewrite the latter function as:
\begin{equation*}
\begin{split}
||\mathcal{W} [{\mathbf {\hat f}}] - \mathcal{W} [{\mathbf g}]||^2_{L^m_{2}} +
\lambda \min\limits_{\substack{S\in L^{n\times m}_2({\mathbb R}^n) \\ rank(O_S) \leq k}}||S_{\mathbf g}-S||^2_{L^{n\times m}_2({\mathbb R}^n)}
\end{split}
\end{equation*}
and our problem can be seen as a task in which we optimize over 2 objects:
\begin{equation*}
\begin{split}
\Phi({\mathbf g}, S) \rightarrow \min_{\substack{{\mathbf g}\in L_{I,2} \\ S\in L^{n\times m}_2({\mathbb R}^n): rank(O_S) \leq k}}
\end{split}
\end{equation*}
where $\Phi({\mathbf g}, S) = ||\mathcal{W} [{\mathbf {\hat f}}] \hspace{-2pt}-\hspace{-2pt} \mathcal{W} [{\mathbf g}]||^2_{L^m_{2}} \hspace{-2pt}+\hspace{-2pt} \lambda ||S_{\mathbf g}\hspace{-2pt}-\hspace{-2pt}S||^2_{L^{n\times l}_2({\mathbb R}^n)}$.

A natural norm on $L_{I,2}$ can be defined as:
\begin{equation*}
\begin{split}
||{\mathbf f}_1||_{L_{I,2}} = \sqrt{(2\pi)^n||{\mathbf f}_1||^2_{L^m_2}+\lambda Tr(\mathcal{M}_{{\mathbf f}_1})} = \\
\sqrt{\int_{{\mathbb R}^n} ((2\pi)^n+\lambda |{\mathbf x}|^2) |{\mathbf f}_1({\mathbf x})|^2 d{\mathbf x}}
\end{split}
\end{equation*}
The naturality of that norm is due to the following property:
\begin{theorem} \label{dense}
$|\Phi({\mathbf g}_1, S) - \Phi({\mathbf g}_2, S)|\leq 
||{\mathbf g}_1-{\mathbf g}_2||_{L_{I,2}}\left( 2 \sqrt{\Phi({\mathbf g}_2, S)}+ ||{\mathbf g}_1-{\mathbf g}_2||_{L_{I,2}}\right)$
\end{theorem}
\begin{proof}[Proof of theorem~\ref{dense}]
Triangle inequality with subsequent Young convolution theorem gives us:
\begin{equation*}
\begin{split}
||\mathcal{W} [{\mathbf {\hat f}}] - \mathcal{W} [{\mathbf g}_1]||_{L^m_2} \leq ||\mathcal{W} [{\mathbf {\hat f}}] - \mathcal{W} [{\mathbf g}_2]||_{L^m_2} + \\ ||\gamma\ast [{\mathbf g}_2\hspace{-2pt}-\hspace{-1pt}{\mathbf g}_1]||_{L^m_2} \hspace{-2pt}\leq\hspace{-1pt} ||\mathcal{W} [{\mathbf {\hat f}}] \hspace{-2pt}-\hspace{-1pt} \mathcal{W} [{\mathbf g}_2]||_{L^m_2}\hspace{-2pt}+\hspace{-1pt} ||\gamma||_{L_1}||{\mathbf g}_2\hspace{-2pt}-\hspace{-1pt}{\mathbf g}_1||_{L^m_2}
\end{split}
\end{equation*}
It is easy to check that $||\gamma||_{L_1} = \int_{{\mathbb R}^n}e^{-|{\mathbf x}|^2/2} d {\mathbf x} = \sqrt{2\pi}^n$ and after squaring both sides we obtain:
\begin{equation}\label{part1}
\begin{split}
||\mathcal{W} [{\mathbf {\hat f}}] - \mathcal{W} [{\mathbf g}_1]||^2_{L^m_2} \leq ||\mathcal{W} [{\mathbf {\hat f}}] - \mathcal{W} [{\mathbf g}_2]||^2_{L^m_2} + \\ (2\pi)^n||{\mathbf g}_2 \hspace{-2pt}-\hspace{-1pt}{\mathbf g}_1||^2_{L^m_2}  \hspace{-2pt}+\hspace{-1pt} 2||\mathcal{W} [{\mathbf {\hat f}}]  \hspace{-2pt}-\hspace{-1pt} \mathcal{W} [{\mathbf g}_2]||_{L^m_2}\sqrt{2\pi}^n||{\mathbf g}_2 \hspace{-2pt}-\hspace{-1pt}{\mathbf g}_1||_{L^m_2}
\end{split}
\end{equation}
Now we again apply triangle inequality to bound for the second part of $\Phi({\mathbf g}_1,S)$:
\begin{equation*}
\begin{split}
||S_{{\mathbf g}_1}-S||_{L^{n\times m}_2({\mathbb R}^n)} = || {\mathbf x}{\mathbf g}^T_1({\mathbf x})  - S({\mathbf x}) ||_{L^{n\times m}_2 ({\mathbb R}^{n})} \leq \\ 
||{\mathbf x}{\mathbf g}^T_1({\mathbf x}) \hspace{-2pt} - \hspace{-1pt} {\mathbf x}{\mathbf g}^T_2({\mathbf x})||_{L^{n\times m}_2 ({\mathbb R}^{n})} \hspace{-3pt}+\hspace{-3pt}
||{\mathbf x}{\mathbf g}^T_2({\mathbf x}) \hspace{-2pt} - \hspace{-1pt} S({\mathbf x})||_{L^{n\times m}_2 ({\mathbb R}^{n})} = \\
\sqrt{\int_{{\mathbb R}^n} |{\mathbf x}|^2\left|{\mathbf g}_1({\mathbf x})  - {\mathbf g}_2({\mathbf x}) \right|^2 d {\mathbf x}} + ||S_{{\mathbf g}_2}-S||_{L^{n\times m}_2({\mathbb R}^n)} 
\end{split}
\end{equation*}
Squaring gives:
\begin{equation*}
\begin{split}
||S_{{\mathbf g}_1}-S||^2_{L^{n\times m}_2({\mathbb R}^n)} \leq \\
\int_{{\mathbb R}^n} |{\mathbf x}|^2\left|{\mathbf g}_1({\mathbf x})  - {\mathbf g}_2({\mathbf x}) \right|^2 d {\mathbf x} + ||S_{{\mathbf g}_2}-S||^2_{L^{n\times m}_2({\mathbb R}^n)} + \\
2 \sqrt{\int_{{\mathbb R}^n} |{\mathbf x}|^2\left|{\mathbf g}_1({\mathbf x})  - {\mathbf g}_2({\mathbf x}) \right|^2 d {\mathbf x}} \cdot ||S_{{\mathbf g}_2}-S||_{L^{n\times m}_2({\mathbb R}^n)}
\end{split}
\end{equation*}
Adding the latter inequality (multiplied by $\lambda$) to inequality~\ref{part1} gives us:
\begin{equation*}
\begin{split}
\Phi({\mathbf g}_1, S) \leq \Phi({\mathbf g}_2, S) + (2\pi)^n||{\mathbf g}_2-{\mathbf g}_1||^2_{L^m_2} + \\  
\lambda \int_{{\mathbb R}^n} |{\mathbf x}|^2\left|{\mathbf g}_1({\mathbf x})  - {\mathbf g}_2({\mathbf x}) \right|^2 d {\mathbf x} + \\
2||\mathcal{W} [{\mathbf {\hat f}}] - \mathcal{W} [{\mathbf g}_2]||\sqrt{2\pi}^n||{\mathbf g}_2-{\mathbf g}_1||_{L^m_2} + \\
2 \lambda||S_{{\mathbf g}_2}-S||_{L^{n\times m}_2({\mathbb R}^n)} \cdot \sqrt{\int_{{\mathbb R}^n} |{\mathbf x}|^2\left|{\mathbf g}_1({\mathbf x})  - {\mathbf g}_2({\mathbf x}) \right|^2 d {\mathbf x}} \\ \leq \Phi({\mathbf g}_2, S) + ||{\mathbf g}_1-{\mathbf g}_2||^2_{L_{I,2}} +
2 \sqrt{\Phi({\mathbf g}_2, S)} ||{\mathbf g}_1-{\mathbf g}_2||_{L_{I,2}} 
\end{split}
\end{equation*}
At the last step we applied inequality $ab+cd\leq \sqrt{a^2+c^2} \sqrt{b^2+d^2}$ for $a = ||\mathcal{W} [{\mathbf {\hat f}}] - \mathcal{W} [{\mathbf g}_2]||$, $b=\sqrt{2\pi}^n||{\mathbf g}_2-{\mathbf g}_1||_{L^m_2}$, $c=\sqrt{\lambda}||S_{{\mathbf g}_2}-S||_{L^{n\times m}_2({\mathbb R}^n)}$ and $d=\sqrt{\lambda\int_{{\mathbb R}^n} |{\mathbf x}|^2\left|{\mathbf g}_1({\mathbf x})  - {\mathbf g}_2({\mathbf x}) \right|^2 d {\mathbf x}}$.
\end{proof}

The simplest idea for an optimization is to minimize over ${\mathbf g}\in L_{I,2}$ and over $S\in L^{n\times m}_2({\mathbb R}^n): rank(O_S) \leq k$ alternatingly. Obviously, the first part would be an optimization over infinite-dimensional object, which cannot be implemented in practice. In order to avoid infiniteness, we will fix a proper class of functions $\mathfrak{M}\subseteq L_{I,2}$ and optimize over $\mathfrak{M}$. A general scheme of optimization is given in the algorithm~\ref{alternate}. Note that we defined ${\mathbf g}_t$ at step 4 of the general scheme as a result of minimization over $L_{I,2}$. This was done for the purposes of theoretical analysis that we provide. In practice instead of steps 4-5 we define ${\mathbf h}_{t}$ as a result of minimization of $\Phi({\mathbf h}, S_{t-1})$ over $\mathfrak{M}$.

\begin{algorithm}
\caption{Alternating scheme}\label{alternate}
\begin{algorithmic}[1]
\Procedure{}{}
\State $S_0 \gets 0$
\For {$t = 1, \cdots, N$}
\State ${\mathbf g}_{t}\gets\arg\min\limits_{{\mathbf g}\in L_{I,2}} \Phi({\mathbf g}, S_{t-1})$
\State Find ${\mathbf h}_{t}\in \mathfrak{M}$ s.t. $||{\mathbf g}_{t}-{\mathbf h}_{t}||_{L_{I,2}}<\varepsilon$
\State $S_{t} \gets \arg\min_{S\in L^{n\times m}_2({\mathbb R}^n): rank(O_S) \leq k} \Phi({\mathbf h}_t, S)$
\EndFor
\EndProcedure
\end{algorithmic}
\end{algorithm}

Note that step 6 of our algorithm is equivalent to minimizing $ ||S_{{\mathbf h}_t}-S||^2_{L^{n\times m}_2({\mathbb R}^n)}$ over $S\in L^{n\times m}_2({\mathbb R}^n): rank(O_S) \leq k$. In the previous section we have already described an optimal solution for that task (equation~\eqref{optS}): $S_t = P_{{\mathbf h}_t} {\mathbf x} {\mathbf h}_t({\mathbf x})^T$ where $P_{{\mathbf h}_t}\in {\mathbb R}^{n\times n}$ is a projection operator that projects to first $k$ principal components of $\mathcal{M}_{{\mathbf h}_t} = \int_{{\mathbb R}^n} {\mathbf x}{\mathbf x}^T |{\mathbf h}_t({\mathbf x})|^2 d{\mathbf x}$. The hardest part of that step is to estimate the matrix $\mathcal{M}_{{\mathbf h}_t}$ for a given ${\mathbf h}_t\in \mathfrak{M}$. Thus, a practical implementation of our algorithm would require $\mathfrak{M}$ to be defined in such a way that the latter integral can be calculated either analytically or numerically. Yet at the same time, in order to fulfill the step 5, $\mathfrak{M}$ should be rich enough in order to approximate functions from $L_{I,2}$ in terms of the natural norm on $L_{I,2}$. By theorem~\ref{dense}, if at step 5 we find ${\mathbf h}_{k}$ such that $||{\mathbf g}_{k}-{\mathbf h}_{k}||_{L_{I,2}}<\varepsilon$, then $|\Phi({\mathbf g}_{k}, S_{k-1})-\Phi({\mathbf h}_{k}, S_{k-1})| \leq \varepsilon (2\sqrt{\Phi({\mathbf g}_{k}, S_{k-1})}+\varepsilon)$. I.e. the objective's value increases only slightly when we substitute ${\mathbf h}_{k}$ for ${\mathbf g}_{k}$. Thus, to summarize, $\mathfrak{M}$ should be:
\begin{itemize}
\item Dense in $L_{I,2}$ (w.r.t. the natural norm $||\cdot||_{L_{I,2}}$).
\item Any ${\mathbf h}\in \mathfrak{M}$ should be given in such a form that $\mathcal{M}_{{\mathbf h}}$ is efficiently computable.
\end{itemize}
An example of $\mathfrak{M}$ that satisfies the latter 2 conditions will be given in the next section.

\subsection{Return to initial coordinates}
One of difficulties in solving~\eqref{fourier-two} via scheme~\ref{alternate} in applications is that it assumes that ${\mathbf f}' = \sqrt{2\pi}^n\mathcal{F}[\sqrt{p}{\mathbf f}]$ is already given to us. In fact, a practical calculation of the Fourier transform $\mathcal{F}[\sqrt{p}{\mathbf f}]$, if $n > 10$, is a problem that can be solved only for special cases of functions ${\mathbf f}$. I.e., in applications it is desirable that an algorithm for the problem deals with functions in the initial coordinate space, rather than in the frequency space. The specifics of our scheme~\ref{alternate} is that it allows a reformulation with initial coordinates.

Indeed, at step 4 of the algorithm we minimize the objective:
$$
\Phi({\mathbf g}, S) = ||\mathcal{W} [{\mathbf {\hat f}}] - \mathcal{W} [{\mathbf g}]||^2_{L^m_{2}} +\lambda ||S_{\mathbf g}-S_{t-1}||^2_{L^{n\times m}_2({\mathbb R}^n)}
$$
where $S_{t-1}({\mathbf x}) = P_{{\mathbf h}_{t-1}} {\mathbf x} {\mathbf h}_{t-1}({\mathbf x})^T$ has been calculated on the previous iteration.
According to theorem~\ref{represent} and formula~\eqref{optS} and using that $J({\mathbf c}) = {\mathbf c}$ we can rewrite the objective as:
$$
||\mathcal{W} [{\mathbf {\hat f}}] - \mathcal{W} [{\mathbf g}]||^2_{L^m_{2}} +\lambda\int_{{\mathbb R}^n} \left|\left| {\mathbf x} {\mathbf g}({\mathbf x})^T\hspace{-3pt} -\hspace{-2pt} P_{{\mathbf h}_{k-1}} {\mathbf x} {\mathbf h}_{k-1}({\mathbf x})^T \right|\right|_F^2 d {\mathbf x} 
$$ 
Using unitarity of the inverse Fourier tranform, we can apply it to $\mathcal{W} [{\mathbf {\hat f}}] - \mathcal{W} [{\mathbf g}]$ and use that $\mathcal{F}^{-1}[\mathcal{W} [{\mathbf {\hat f}}]] = \frac{1}{\sqrt{2\pi}^n}e^{-|{\mathbf x}|^2/2} {\mathbf f}$, $\mathcal{F}^{-1}[\mathcal{W} [{\mathbf g}]] = \frac{1}{\sqrt{2\pi}^n}e^{-|{\mathbf x}|^2/2} {\mathbf G}$ where ${\mathbf G} = \mathcal{F}^{-1}[{\mathbf g}]$. Recall that an image of $L_{I,2}$ under inverse Fourier transform is the cartesian power of Sobolev space $\left(W^{1,2}\right)^m$, therefore ${\mathbf G} \in \left(W^{1,2}\right)^m$. Using a standard property of $W^{1,2}$, i.e. the duality between the coordinate operator and the differentiation, we obtain that $\mathcal{F}^{-1}[{\mathbf x} {\mathbf g}({\mathbf x})^T] = -i \partial_{{\mathbf x}} {\mathbf G}^T$ and $\mathcal{F}^{-1}[P_{{\mathbf h}_{t-1}} {\mathbf x} {\mathbf h}_{t-1}({\mathbf x})^T ] = -iP_{{\mathbf h}_{t-1}} \partial_{{\mathbf x}} {\mathbf H}^T_{t-1}$ where ${\mathbf H}_{t-1} = \mathcal{F}^{-1}[{\mathbf h}_{t-1}]$. Thus, at step 4 we solve:
\begin{equation}\label{incoord}
\begin{split}
\Phi'_{t-1}({\mathbf G}) = \frac{1}{(2\pi)^n}||e^{-|{\mathbf x}|^2/2} {\mathbf f} - e^{-|{\mathbf x}|^2/2} {\mathbf G}||^2_{L^m_{2}} +\\
\lambda \int_{{\mathbb R}^n} \left|\left| \frac{\partial {\mathbf G}}{{\partial {\mathbf x}}} - \frac{\partial {\mathbf H}_{t-1}}{\partial {\mathbf x}} P_{t-1}  \right|\right|_F^2 d {\mathbf x} \rightarrow \min_{{\mathbf G}\in \left(W^{1,2}\right)^m}
\end{split}
\end{equation}
where $\frac{\partial {\mathbf F}}{{\partial {\mathbf x}}} = \begin{bmatrix}
\frac{\partial F_i}{\partial x_j}
\end{bmatrix}_{1\leq i \leq m, 1\leq j \leq n}$ is Jacobian and $P_{t-1} = P_{{\mathbf h}_{t-1}}$.

It is easy to see that the natural norm on $\left(W^{1,2}\right)^m$ (dual to the norm on $L_{I,2}$) is:
$$
||{\mathbf F}||_{\left(W^{1,2}\right)^m} = \sqrt{\int_{{\mathbb R}^n} \left\{(2\pi)^n |{\mathbf F}|^2 + \lambda |\frac{\partial {\mathbf F}}{{\partial {\mathbf x}}}|^2\right\}d {\mathbf x}}
$$
The set dual to $\mathfrak{M}$ is defined as $\mathfrak{M}' = \{\mathcal{F}^{-1}[{\mathbf h}] | {\mathbf h}\in \mathfrak{M}\}$.
The matrix $\mathcal{M}_{t} = \mathcal{M}_{{\mathbf h}_{t}}$ also can be defined using ${\mathbf H}_{t}$ only:
\begin{equation*}
\begin{split}
\mathcal{M}_{t} = \begin{bmatrix}
\langle x_i {\mathbf h}_{t}, x_j {\mathbf h}_{t} \rangle_{L^m_2({\mathbb R}^n)}\end{bmatrix}_{n\times n} = \begin{bmatrix} 
\langle \frac{\partial {\mathbf H}_{t}}{\partial x_i}, \frac{\partial {\mathbf H}_{t}}{\partial x_j} \rangle_{L^m_2({\mathbb R}^n)}\end{bmatrix}_{n\times n}  \\ =
\int_{{\mathbb R}^n} \frac{\partial {\mathbf H}_{t}}{\partial {\mathbf x}}^T \frac{\partial {\mathbf H}_{t}}{\partial {\mathbf x}}  d {\mathbf x}
\end{split}
\end{equation*}
Thus, the algorithm~\ref{alternate1} is dual to~\ref{alternate}.
\begin{algorithm}
\caption{Alternating scheme with initial coordinates}\label{alternate1}
\begin{algorithmic}[1]
\Procedure{}{}
\State $P_0, {\mathbf H}_{0} \gets 0$
\For {$t = 1, \cdots, N$}
\State ${\mathbf G}_{t} \gets \arg\min_{{\mathbf G}\in \left(W^{1,2}\right)^m} \Phi'_{t-1}({\mathbf G})$
\State Find ${\mathbf H}_{t}\in \mathfrak{M'}$ s.t. $||{\mathbf G}_{t} - {\mathbf H}_{t}||_{\left(W^{1,2}\right)^m} \leq \varepsilon$
\State $\mathcal{M}_{t} \gets\int_{{\mathbb R}^n} \frac{\partial {\mathbf H}_{t}}{\partial {\mathbf x}}^T \frac{\partial {\mathbf H}_{t}}{\partial {\mathbf x}}  d {\mathbf x}$
\State $P_{t} \gets $ projection to first $k$ principal components of $\mathcal{M}_{t}$ 
\EndFor
\EndProcedure
\end{algorithmic}
\end{algorithm}

Let us now give an example of a set $\mathfrak{M}$ that satisfies both conditions that we imposed in the previous section. Instead of defining $\mathfrak{M}$ we will define its dual $\mathfrak{M}' = \left\{{\mathbf H} = (H_1, \cdots, H_m)| H_i\in FF\right\}$ where $FF$ is a set of functions of the following form:
$$
\sigma(\gamma(R-|{\mathbf x}|))\sum_{i=1}^M \psi({\mathbf a}^T_i {\mathbf x}-b_i)
$$
where ${\mathbf a}_i\in {\mathbb R}^n, b_i \in {\mathbb R},\gamma, R\in {\mathbb R}_+$ are parameters. The function $\sigma$ is the standard sigmoid $\sigma(x) = \frac{1}{1+e^{-x}}$, whereas for $\psi$ we only assume that it is some non-constant function whose first derivatives are continuous and bounded.

\begin{theorem}\label{hornik}
 Let $\psi: {\mathbb R}^n\rightarrow {\mathbb R}$ be a non-constant function whose first derivatives are continuous and bounded and ${\mathbf G}\in \left(W^{1,2}\right)^m$ is fixed. Then, for any $\epsilon>0$ there exists ${\mathbf G}_\epsilon\in \mathfrak{M}'$ such that $||{\mathbf G}-{\mathbf G}_\epsilon||_{\left(W^{1,2}\right)^m} < \epsilon$.
\end{theorem}
\begin{proof} [Proof of theorem~\ref{hornik}] In our proof we will use the following result from~\cite{Hornik}.
\begin{theorem}[K. Hornik] Let $\psi: {\mathbb R}^n\rightarrow {\mathbb R}$ be a non-constant function whose first derivatives are continuous and bounded, $\mu$ is a finite measure on ${\mathbb R}^n$ and $g\in C_c^\infty ({\mathbb R}^n)$ is fixed. Then, for any $\epsilon>0$, there exists a function of the form $\theta ({\mathbf x}) = \sum_{i=1}^M \psi({\mathbf a}^T_i {\mathbf x}-b_i)$ such  that 
$$||\theta - g||_{\mu} +  ||\frac{\partial \theta}{\partial {\mathbf x}} - \frac{\partial g}{\partial {\mathbf x}}||_{\mu} < \epsilon$$
where $||h||_\mu = \sqrt{\int_{{\mathbb R}^n} |h({\mathbf x})|^2 d\mu}$.
\end{theorem}
W.l.o.g. we assume that $m=1$.
Since $W^{1,2}$ is the completion of $C_c^\infty ({\mathbb R}^n)$, it is enough to prove that $\mathfrak{M}'$ is dense in $C_c^\infty ({\mathbb R}^n)$. By the latter we mean that if $g\in C_c^\infty ({\mathbb R}^n)$ is given, then for any $\epsilon>0$ one can find ${\mathbf a}_i\in {\mathbb R}^n, b_i \in {\mathbb R}, i =\overline{1,M}, \gamma, R\in {\mathbb R}_+$ such that for $\phi ({\mathbf x}) = \sigma(\gamma(R-|{\mathbf x}|))\sum_{i=1}^{M} \psi({\mathbf a}_i^T {\mathbf x}-b_i)$ the following holds:
$$
||g-\phi||_{L_2({\mathbb R}^n)} + ||\frac{\partial g}{\partial {\mathbf x}}-\frac{\partial \phi}{\partial {\mathbf x}}||_{L_2({\mathbb R}^n)} < \epsilon
$$
Note that we used slightly different metrics on $W^{1,2}$ which is known to be equivalent to our natural metrics in terms of induced topology.

First let us prove some general bounds on the distance between $g$ and $\phi ({\mathbf x}) = \sigma(\gamma(R-|{\mathbf x}|))\theta ({\mathbf x})$, where $\theta ({\mathbf x}) = \sum_{i=1}^{M} \psi({\mathbf a}^T_i {\mathbf x}-b_i)$. For brevity we will write $\sigma, \sigma'$ instead of $\sigma(\gamma(R-|{\mathbf x}|)), \sigma'(\gamma(R-|{\mathbf x}|))$. 
The derivative of $\phi$ is:
$$
\frac{\partial \phi}{\partial {\mathbf x}} = -\frac{\gamma {\mathbf x}}{|{\mathbf x}|} \sigma' \theta + \sigma \frac{\partial \theta}{\partial {\mathbf x}} = -\frac{\gamma {\mathbf x}}{|{\mathbf x}|} \sigma (1-\sigma) \theta + \sigma \frac{\partial \theta}{\partial {\mathbf x}}
$$
If $\gamma, R$ given, then we can define a finite measure $\mu$ by $d\mu = \sigma(\gamma(R-|{\mathbf x}|))^2 d {\mathbf x}$. The following inequalities hold:
\begin{equation*}
\begin{split}
||\sigma \theta - g||_{L_2} + ||-\frac{\gamma {\mathbf x}}{|{\mathbf x}|} \sigma (1-\sigma) \theta + \sigma \frac{\partial \theta}{\partial {\mathbf x}} - \frac{\partial g}{\partial {\mathbf x}}||_{L_2} \leq \\
||\sigma \theta - \sigma g||_{L_2} + ||\sigma g - g||_{L_2} + ||-\frac{\gamma {\mathbf x}}{|{\mathbf x}|} \sigma (1-\sigma) \theta ||_{L_2} + \\ ||\sigma \frac{\partial \theta}{\partial {\mathbf x}} - \frac{\partial g}{\partial {\mathbf x}}||_{L_2} \leq 
||\theta - g||_{\mu} + ||(1-\sigma) g||_{L_2} + \\
\gamma||(1-\sigma) (\theta-g+g) ||_{\mu} + ||\sigma \frac{\partial \theta}{\partial {\mathbf x}} - \sigma \frac{\partial g}{\partial {\mathbf x}}||_{L_2} + \\ ||\sigma \frac{\partial g}{\partial {\mathbf x}} - \frac{\partial g}{\partial {\mathbf x}}||_{L_2} \leq 
||\theta - g||_{\mu} + ||(1-\sigma) g||_{L_2} + \\ 
\gamma||(1-\sigma) (\theta-g)||_{\mu} +\gamma|| (1-\sigma)g ||_{\mu} + ||\frac{\partial \theta}{\partial {\mathbf x}} - \frac{\partial g}{\partial {\mathbf x}}||_{\mu}  +\\ 
||\frac{\partial g}{\partial {\mathbf x}} (1 - \sigma)||_{L_2} \leq 
\alpha + \beta
\end{split}
\end{equation*}
where $\alpha = ||\theta - g||_{\mu}(1+\gamma) +  ||\frac{\partial \theta}{\partial {\mathbf x}} - \frac{\partial g}{\partial {\mathbf x}}||_{\mu}$ and $\beta = ||(1-\sigma) g||_{L_2} + 
\gamma|| (1-\sigma)g ||_{\mu}  + ||\frac{\partial g}{\partial {\mathbf x}} (1 - \sigma)||_{L_2} $.

It is easy to see that $\beta$ does not depend on parameters ${\mathbf a}_i\in {\mathbb R}^n, b_i \in {\mathbb R}$. Let us prove that for any fixed $\gamma>0$ we have:
$$
\lim_{R\rightarrow +\infty} \beta = 0
$$
Since $g\in C_c^\infty ({\mathbb R}^n)$ has a compact support, then we can assume that there is $\rho, C>0$ such that $g({\mathbf x}) = 0, |{\mathbf x}|>\rho$ and $\forall {\mathbf x}\,\,|g({\mathbf x})| \leq C, |\frac{\partial g}{\partial {\mathbf x}}| \leq C$.  Then if $R>\rho$ we can bound $\beta$:
\begin{equation*}
\begin{split}
\beta \leq C (2+\gamma) \sqrt{\int_{|{\mathbf x}|\leq \rho} (1-\sigma(\gamma(R-|{\mathbf x}|)))^2 d {\mathbf x}} = \\
C (2+\gamma) \sqrt{\int_{|{\mathbf x}|\leq \rho} \frac{1}{(1+e^{\gamma(R-|{\mathbf x}|)})^2} d {\mathbf x}} \leq \\
C (2+\gamma) \sqrt{\frac{\int_{|{\mathbf x}|\leq \rho} d {\mathbf x}}{(1+e^{\gamma(R-\rho)})^2}} = \\
C (2+\gamma)  \frac{\pi^{n/4} \rho^{n/2}}{\sqrt{\Gamma(\frac{n}{2}+1)} (1+e^{\gamma(R-\rho)})} \mathop\rightarrow\limits^{R\rightarrow \infty} 0 \end{split}
\end{equation*}
Thus, for any $\epsilon>0$ we can find $R_{\gamma,\epsilon}$ such that $\beta < \frac{\epsilon}{2}$ whenever $R>R_{\gamma,\epsilon}$. Let us fix such $R$. Thus, the measure $\mu$ is defined.  Now, by Hornik's result, single layer feedforward neural networks are dense in $C^{1,2} (\mu)$, therefore we can find ${\mathbf a}_i\in {\mathbb R}^n, b_i \in {\mathbb R}$ such that $\alpha < \frac{\epsilon}{2}$. This implies that $||\sigma \theta - g||_{L_2} + ||\frac{\partial (\sigma \theta)}{\partial {\mathbf x}} - \frac{\partial g}{\partial {\mathbf x}}||_{L_2} \leq \epsilon$. I.e. $\mathfrak{M}$ is dense in $C_c^\infty ({\mathbb R}^n)$ and $W^{1,2}$.
\end{proof}

\subsection{Practical algorithm with initial coordinates}
Let us assume for simplicity that $m=1$. If we fix $R, \gamma$ and $M$ we obtain a subset of class $\mathfrak{M}'$, denoted $\mathfrak{M}'_{R, \gamma, M}$. 
With a goal to implement our scheme as a practical algorithm, instead of an optimization over $W^{1,2}$ at step 4 of algorithm~\ref{alternate1} we will optimize over $\mathfrak{M}'_{R, \gamma, M}$. Also, step 5 is not needed at all and ${\mathbf G}_{t} = {\mathbf H}_{t}$.

In practice it is natural to approximate the first part of~\ref{incoord} as:
\begin{equation*}
\begin{split}
\propto \Phi^1_{t-1} = \frac{1}{K}\sum_{i=1}^K ||{\mathbf f} ({\mathbf x}_i) - {\mathbf G}({\mathbf x}_i)||^2 
\end{split}
\end{equation*}
where ${\mathbf x}_1, \cdots, {\mathbf x}_K \sim \frac{1}{\sqrt{\pi}^n}e^{-|{\mathbf x}|^2}$.

Since ${\mathbf G}, {\mathbf G}_{t-1}\in \mathfrak{M}'_{R, \gamma, M}$, i.e. ${\mathbf G} = \sigma(\gamma(R-|{\mathbf x}|))\theta_{\mathbf G} ({\mathbf x})$ and ${\mathbf G}_{t-1} = \sigma(\gamma(R-|{\mathbf x}|))\theta_{{\mathbf G}_{t-1}} ({\mathbf x})$, then second part of~\ref{incoord} becomes: 
\begin{equation*}
\begin{split}
\int_{{\mathbb R}^n} \left|\left| \frac{\partial {\mathbf G}}{{\partial {\mathbf x}}} - \frac{\partial {\mathbf G}_{t-1}}{\partial {\mathbf x}} P_{t-1}  \right|\right|_F^2 d {\mathbf x} = ||-\gamma\frac{(1-\sigma){\mathbf x}}{|{\mathbf x}|}\theta_{\mathbf G} ({\mathbf x})+\\ \frac{\partial \theta_{\mathbf G}}{\partial {\mathbf x}}^T-\left(-\gamma\frac{(1-\sigma){\mathbf x}}{|{\mathbf x}|}\theta_{{\mathbf G}_{t-1}} ({\mathbf x})+\frac{\partial \theta_{{\mathbf G}_{t-1}}}{\partial {\mathbf x}}^T\right)  P_{t-1}||^2_{\mu} \end{split}
\end{equation*}
where $d\mu = \sigma^2 d {\mathbf x}$. Thus, it can be approximated as:
\begin{equation*}
\begin{split}
\propto \Phi^2_{t-1} = \frac{1}{L}\sum_{i=1}^L ||-\gamma\frac{(1-\sigma){\mathbf z}_i}{|{\mathbf z}_i|}\theta_{\mathbf G} ({\mathbf z}_i)+ \frac{\partial \theta_{\mathbf G} ({\mathbf z}_i)}{\partial {\mathbf x}}^T-\\ \left(-\gamma\frac{(1-\sigma){\mathbf z}_i}{|{\mathbf z}_i|}\theta_{{\mathbf G}_{t-1}} ({\mathbf z}_i)+\frac{\partial \theta_{{\mathbf G}_{t-1}}({\mathbf z}_i)}{\partial {\mathbf x}}^T\right)  P_{t-1}||^2
\end{split}
\end{equation*}
where ${\mathbf z}_1, \cdots, {\mathbf z}_L \sim \sigma(\gamma(R-|{\mathbf x}|))^2$.

Since
\begin{equation*}
\begin{split}
\mathcal{M}_{t} =\int_{{\mathbb R}^n} \frac{\partial {\mathbf G}_{t}}{\partial {\mathbf x}}^T \frac{\partial {\mathbf G}_{t}}{\partial {\mathbf x}}  d {\mathbf x} = \int (-\gamma\frac{(1-\sigma){\mathbf x}}{|{\mathbf x}|}\theta_{{\mathbf G}_{t}} ({\mathbf x})+\\
\frac{\partial \theta_{{\mathbf G}_{t}}}{\partial {\mathbf x}}^T )(-\gamma\frac{(1-\sigma){\mathbf x}}{|{\mathbf x}|}\theta_{{\mathbf G}_{t}} ({\mathbf x})+
\frac{\partial \theta_{{\mathbf G}_{t}}}{\partial {\mathbf x}}^T )^T d\mu
\end{split}
\end{equation*}
it is natural to estimate it as:
\begin{equation*}
\begin{split}
\hat{\mathcal{M}}_{t} = \frac{1}{L}\sum_{i=1}^L (-\gamma\frac{(1-\sigma){\mathbf z}_i}{|{\mathbf z}_i|}\theta_{{\mathbf G}_{t}} ({\mathbf z}_i)+
\frac{\partial \theta_{{\mathbf G}_{t}}({\mathbf z}_i)}{\partial {\mathbf x}}^T )\\(-\gamma\frac{(1-\sigma){\mathbf z}_i}{|{\mathbf z}_i|}\theta_{{\mathbf G}_{t}} ({\mathbf z}_i)+
\frac{\partial \theta_{{\mathbf G}_{t}}({\mathbf z}_i)}{\partial {\mathbf x}}^T )^T 
\end{split}
\end{equation*}
Thus, the pseudocode of the algorithm can be found below.
\begin{algorithm}
\caption{Practical algorithm with initial coordinates}\label{alternate2}
\begin{algorithmic}[1]
\Procedure{}{Parameters: $\lambda, \gamma, R>0, N, M, K, L\in {\mathbb N}$}
\State $P_0, {\mathbf G}_{0} \gets 0$
\State Sample ${\mathbf x}_1, \cdots, {\mathbf x}_K \sim \frac{1}{\sqrt{\pi}^n} e^{-|{\mathbf x}|^2}$ and ${\mathbf z}_1, \cdots, {\mathbf z}_L \sim \sigma(\gamma(R-|{\mathbf x}|))^2$.
\For {$t = 1, \cdots, N$}
\State ${\mathbf G}_{t} \gets \arg\min_{{\mathbf G}\in \mathfrak{M}'_{R, \gamma, M}} \Phi^1_{t-1}({\mathbf G})+\lambda\Phi^2_{t-1}({\mathbf G})$
\State Estimate $\hat{\mathcal{M}}_{t}$
\State $P_{t} \gets $ projection to first $k$ principal components of $\hat{\mathcal{M}}_{t}$ 
\EndFor
\State Output $P_N, {\mathbf G}_{N}$
\EndProcedure
\end{algorithmic}
\end{algorithm}

\subsection{Experiments}
We experimented with the algorithm~\ref{alternate2} setting our parameters as: number of iterations $N=200$, number of sampled points $K = 1000, L = 10000$. Parameters of ``restricted'' neural network model $\mathfrak{M}'_{R, \gamma, M}$ were set as: $M=200, \gamma = +\infty$ (the latter is equivalent to setting $\sigma(\gamma(R-|{\mathbf x}|))^2$ as the uniform distribution over the ball $B_{R}({\mathbf 0})$ and $\gamma(1-\sigma) = 0$ in all formulae). We also set $R = Q(0.01, n)$ where $Q(p, n)$ is the quantile function defined as: ${\mathbb P}_{{\mathbf x} \sim \frac{1}{\sqrt{\pi}^n}e^{-|{\mathbf x}|^2}}[|{\mathbf x}|>Q(p, n)] = p$.

We experimented with two dimensions, $n=4$ and $n=6$, setting our main function as $f(x_1, \cdots, x_n) = A(x_1, x_2) + C[|{\mathbf x}|\leq 1]$ where $A(x,y) = -20\exp [-0.2\sqrt{0.5\left(x^{2}+y^{2}\right)}]-\exp [0.5\left(\cos 2\pi x + \cos 2\pi y \right)] + e + 20$ is the Ackley function and $[|{\mathbf x}|\leq 1]$ is the indicator function of the unit ball. We were interested in $k=2$, as in this case, for any $C$, a correct dimensionality reduction would give $P_N \approx {\mathbf e}_1 {\mathbf e}^T_1+{\mathbf e}_2 {\mathbf e}^T_2$. Thus, a natural measure of an accuracy of our algorithm is ${\bf acc} = ||P_N - {\mathbf e}_1 {\mathbf e}^T_1-{\mathbf e}_2 {\mathbf e}^T_2||_{F}$. How the resulting accuracy depends on the parameter $\lambda$ for different values of $C$ is shown on the graph below.
\begin{figure}[h]
\includegraphics[scale=0.45]{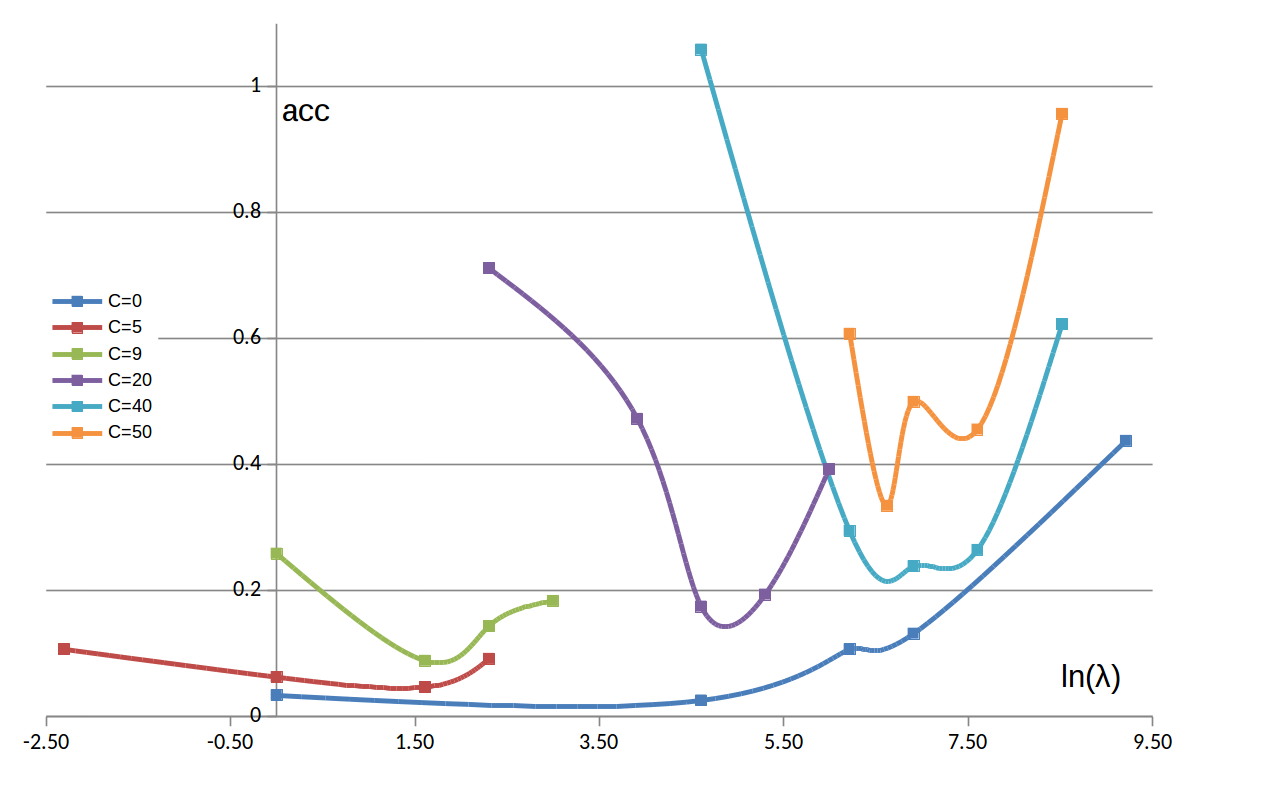}
\end{figure}

It is easy to see that for any fixed $C$ there is an optimal value for $\lambda$. Thus, increasing $\lambda$ does not simply lead to an improved accuracy. We believe the mechanics of that ``over-penalization'' is the following: if $\lambda$ is too large, then the result of the first iteration (specifically of step 5) will be an ``over-smoothed'' function ${\mathbf G}_{1}$, probably a function that behaves almost linearly in $B_{R}({\mathbf 0})$; and at subsequent iterations, due to the large contribution of the second term to our objective, the algorithm fails to jump out of the area around a local minimum (because the second term forces ${\mathbf G}_{t+1}$ to adapt to the gradient field of ${\mathbf G}_{t}$). The latter interpretation is, of course, hypothetical and needs further experimental research. The second observation is rather trivial, if $C$ increases, then an accuracy (for optimal $\lambda$) becomes worse. This happens not because of some fundamental shortcomings of the algorithm, but simply because adapting to the second part of our main function, i.e. $C[|{\mathbf x}|\leq 1]$ for large $C$, requires tuning $R$ to a larger value than we set. 

Overall, this algorithm is the first and the most straightforward way to turn the ``dual'' view of problem~\ref{problem} to a practical solution. More efficient, robust and theoretically substantiated algorithms are a subject of future work.

\bibliographystyle{plain}
\bibliography{lit}

\end{document}